\PassOptionsToPackage{prologue,dvipsnames,table}{xcolor}
\documentclass{article}

\usepackage{amsmath,amsfonts,bm}

\def\eqref#1{equation~\ref{#1}}

\def\1{\bm{1}}

\def\mA{{\bm{A}}}

\def\mK{{\bm{K}}}

\def\mO{{\bm{O}}}

\def\mQ{{\bm{Q}}}

\def\mS{{\bm{S}}}

\def\mV{{\bm{V}}}
\def\mW{{\bm{W}}}
\def\mX{{\bm{X}}}

\DeclareMathAlphabet{\mathsfit}{\encodingdefault}{\sfdefault}{m}{sl}
\SetMathAlphabet{\mathsfit}{bold}{\encodingdefault}{\sfdefault}{bx}{n}

\usepackage{microtype}
\usepackage{graphicx}
\usepackage{subfigure}
\usepackage{booktabs} 
\usepackage{multirow}
\usepackage{hyperref}
\usepackage{amsthm}
\usepackage{thmtools}

\usepackage[table,xcdraw，dvipsnames]{xcolor}

\definecolor{p}{RGB}{200,0,200}

\usepackage{array,multirow,graphicx}
\usepackage{diagbox}
\usepackage[normalem]{ulem}
\useunder{\uline}{\ul}{}

\definecolor{baselinecolor}{gray}{.9}
\newcommand{\baseline}[1]{\cellcolor{baselinecolor}{#1}}
\newcommand\LHS{\mathrm{LHS}}
\newcommand\RHS{\mathrm{RHS}}
\makeatletter\renewcommand\paragraph{\@startsection{paragraph}{4}{\z@}
  {.5em \@plus1ex \@minus.2ex}{-.5em}{\normalfont\normalsize\bfseries}}
\def\@copyrightspace{\relax}  
\makeatother

\usepackage[accepted]{icml2024}

\usepackage{amsmath}
\usepackage{amssymb}
\usepackage{mathtools}
\usepackage{amsthm}
\usepackage[capitalize,noabbrev]{cleveref}

\theoremstyle{plain}
\newtheorem{theorem}{Theorem}[section]

\newtheorem{lemma}[theorem]{Lemma}

\theoremstyle{definition}

\theoremstyle{remark}

\usepackage[textsize=tiny]{todonotes}

\icmltitlerunning{SpikeZIP-TF: Conversion is All You Need for Transformer-based SNN}

\begin{document}

\twocolumn[
\icmltitle{SpikeZIP-TF: Conversion is All You Need for Transformer-based SNN}




\icmlsetsymbol{equal}{*}

\begin{icmlauthorlist}
\icmlauthor{Kang You}{equal,sjtu}
\icmlauthor{Zekai Xu}{equal,sjtu}
\icmlauthor{Chen Nie}{sjtu}
\icmlauthor{Zhijie Deng}{sjtu}
\icmlauthor{Xiang Wang}{comp}
\icmlauthor{Qinghai Guo}{comp}
\icmlauthor{Zhezhi He}{sjtu}
\end{icmlauthorlist}

\icmlaffiliation{sjtu}{School of Electronic Information and Electrical Engineering, Shanghai Jiao Tong University, Shanghai, China}
\icmlaffiliation{comp}{Huawei Technologies, Shenzhen, China}

\icmlcorrespondingauthor{Zhezhi He}{zhezhi.he@sjtu.edu.cn}

\icmlkeywords{Machine Learning, ICML}

\vskip 0.3in
]



\printAffiliationsAndNotice{\icmlEqualContribution} 

\begin{abstract}
Brain-inspired spiking neural network (SNN) has attracted great attention because of its high efficiency over the traditional artificial neural network (ANN). 
Currently, ANN to SNN conversion methods can produce SNNs using convolution neural network as backbone architecture which achieves on-par accuracy w.r.t ANNs with ultra-low latency (\textit{e.g.}, 8 time-steps) on computer vision (CV) tasks.
Although Transformer-based networks have achieved the prevailing precision on both CV and natural language processing (NLP) tasks, Transformer-based SNNs are still lagging behind their ANN counterparts.
In this work, we introduce a novel ANN-to-SNN conversion method, called SpikeZIP-TF, through which ANN and the converted SNN are exactly equivalent thus incurring no accuracy degradation. 
SpikeZIP-TF achieves 83.82\% Top-1 accuracy on the CV image classification task with ImageNet dataset and 93.79\% accuracy on the NLP dataset (SST-2), which both are higher than SOTA Transformer-based SNNs. 
The code is publicly available at: \href{https://github.com/Intelligent-Computing-Research-Group/SpikeZIP-TF}{https://github.com/Intelligent-Computing-Research-Group/SpikeZIP-TF}
\end{abstract}

\section{Introduction}

Spiking neural network (SNN) \cite{maass1997networks} is a type of biologically plausible neural network inspired by brains of living organisms. Unlike modern ANNs \cite{lecun2015deep} using continuous activation value to propagate information between neurons synchronously, SNNs utilize discrete events or ``spikes" for asynchronous neuron-to-neuron communication and processing \cite{merolla2014million, davies2018loihi}.
Meanwhile, in the field of deep learning, Transformers \cite{vaswani2017attention} have made significant strides and revolutionized various applications. Inspired by the architecture of the ANN Transformer, introducing the Transformer structure to SNN to improve the SNN accuracy is an emerging trend \cite{zhou2022spikformer,lv2023spikebert}. 

Currently, methods to train Transformer-based SNN come in twofold: \textit{directly training (DT)} and \textit{ANN-to-SNN Conversion (A2S)} \cite{roy2019towards}.
The DT methods leverage back-propagation through time (BPTT) to update the synaptic weights of SNN. Unfortunately, due to the inaccurate gradient approximation \cite{neftci2019surrogate} for the non-differential SNN neuron, \textit{e.g.}, integrate and fire (IF) neuron, an accuracy gap persists between SNN and its ANN counterpart \cite{zhou2024spikformer,lv2023spikebert,zhou2023spikingformer}. 

\begin{figure}[t]
\centering
\includegraphics[width=\linewidth]{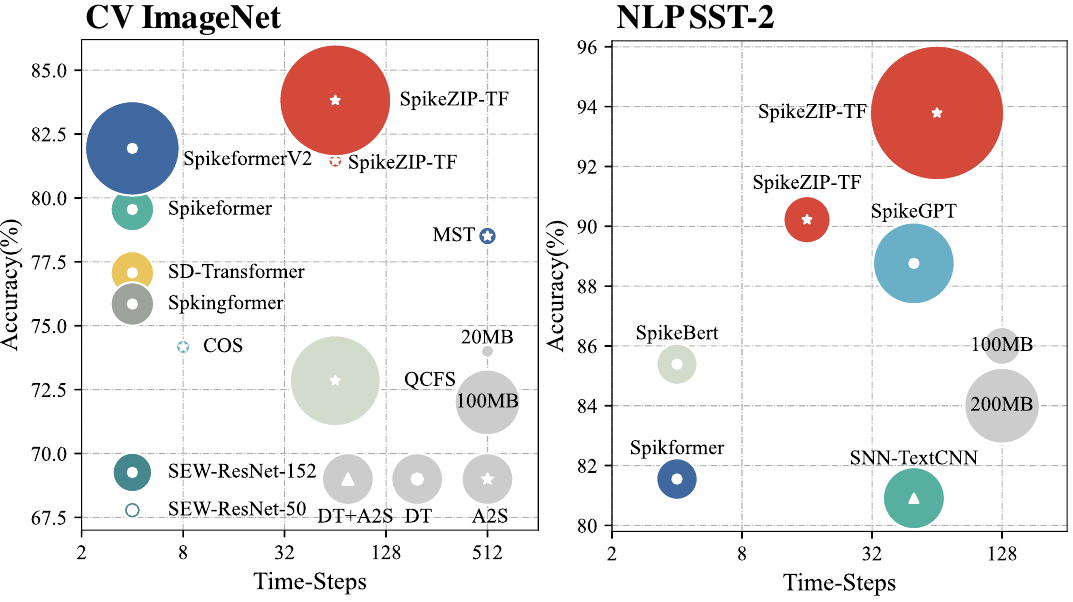}
\caption{\textbf{Comparison of Transformer-based SNNs.} The markers, represented by circles, star, and triangle shapes, denote the direct learning (DT) method, ANN-to-SNN (A2S) conversion method and using both the DT and A2S methods, respectively, where the area of the scatter corresponds to the model size. Results show that the pikeZIP-TF generated SNN achieves higher accuracy with greater model size than the other recent SNNs. 
The largest model size of SpikeZIP-TF on ImageNet is 304.33 MB.
}  
\label{over_page}
\end{figure}

Rather than directly training an SNN, the A2S methods transfer the parameters of the pre-trained ANN to its SNN counterpart \cite{cao2015spiking} that yields close-to-ANN accuracy. The previous A2S algorithm achieves the on-par accuracy to ANN with ultra-low latency (\textit{e.g.}, 8 time-steps) on convolution-based SNN \cite{hu2023fast}.
However, when leveraging the existing A2S algorithms to produce Transformer-based SNNs, it is difficult to build the equivalence between SNN operators and ANN operators, like softmax, layer normalization \cite{ba2016layer}, and attention (\textit{i.e.}, dot product with two non-stationary matrix).
The inequivalence hinders the development of A2S algorithm for Transformer-based SNN.
Correspondingly, we propose SpikeZIP Transformer (\textit{aka.} SpikeZIP-TF), which introduces the spike-equivalent self-attention (SESA), Spike-Softmax and Spike-LayerNorm, while maintaining the equivalence between the operators of ANN and SNN. 
\cref{over_page} shows the overall results of SpikeZIP-TF on CV (ImageNet) and NLP task (SST-2). 

Contributions of SpikeZIP-TF are summarized as follows:
\begin{itemize}
    \item We propose an ANN-SNN conversion method called SpikeZIP-TF that builds the equivalence between the activation-quantized Transformer-based ANN and SNN by supporting the SNN-unfriendly operators of ANN (\textit{e.g.}, softmax and layernorm) in converted SNN.
    
    \item SpikeZIP-TF deals with both the CV and NLP tasks by converting the quantized vision Transformer (ViT) \cite{dosovitskiy2020image} and Roberta \cite{devlin2018bert} to SNN and achieves the state-of-the-art accuracy than competing Transformer-based SNNs.
\end{itemize}

\section{Background and Related Works}

\paragraph{Spiking Neurons.}
Integrate and fire (IF) neuron is widely utilized in the A2S methods due to the mathematical similarity between the IF neuron and ReLU \cite{bu2023optimal}. 
Nevertheless, the error in the accumulated output of the IF neuron to the ReLU persists, which hampers the accuracy of SNNs. 
To deal with the error, a recently emerged SNN neuron, which we name it as bipolar integrate and fire with spike tracer (ST-BIF) neuron, are introduced to further approaching the equivalence to the quantized-ReLU (Q-ReLU) function \cite{li2022quantization,hu2023fast}. 
Unfortunately, the native ST-BIF neuron is merely equivalent to the quantized-ReLU function rather than the quantized function for activation in attention which is widely used in self-attention in Transformer.

\begin{figure}[t]
\centering
\includegraphics[width=\linewidth]{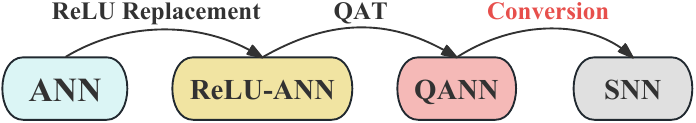}
\caption{The conversion pipeline of SpikeZIP-TF.} 
\label{conversion_pipline}
\end{figure}

\paragraph{Learning Methods of SNN}
comes in twofolds: direct training (DT) and ANN-to-SNN conversion (A2S). The DT algorithm employs the back-propagation through time (BPTT) \cite{lee2016training,shrestha2018slayer} with surrogate gradient \cite{neftci2019surrogate} to train an SNN for a fixed time-step. However, due to the errors of esitamted gradient during training, a loss gap exists between the SNN obtained through DT and its ANN counterpart. 
In the realm of A2S algorithm, the ReLU/Q-ReLU layers in ANN are substituted with spiking neurons, resulting in an equivalent SNN model that achieves comparable accuracy to the ANN \cite{bu2023optimal,hu2023fast,li2022quantization,hao2023bridging,wang2023masked,rueckauer2017conversion}. 
Compared to DT, A2S-based SNNs not only achieve higher accuracy, but also consume lower training cost in terms of time and memory. 
Such characteristic makes the SNN more amenable to model scaling and deployment on neuromorphic hardware.

As depicted in \cref{conversion_pipline}, SpikeZIP-TF adheres to the conversion pipeline established by prior A2S methods \cite{bu2023optimal,li2022quantization,hu2023fast}. 
The process initiates with the replacement of activation functions in the ANN to ensure that only ReLUs are present. 
Subsequently, quantization-aware training (QAT) \cite{gholami2022survey,he2019simultaneously} is applied to acquire a low bit-width and high-accuracy QANN.
Finally, the QANN undergoes conversion to an SNN by substituting the ReLU activation(neuron) with specific spiking neuron, without accuracy degradation.

\begin{table}[t]
\centering
\caption{\textbf{Summary of Transformer-based SNNs.} SSA: spike self-attention; SDSA: spike-driven self-attention; ASR-SA: average spiking rate self-attention; S-RWKV denotes spiking-RWKV; VSA: vanilla self-attetion; QVAS: quantized vanilla self-attetion.}
\label{related_work}
\resizebox{\linewidth}{!}{
\begin{tabular}{@{}rcccccc@{}}
\toprule
Methods & Algorithm & Neuron & Attention & NAS & Pretrain & Distill \\ \midrule
SpikformerV1 & BPTT & LIF & SSA &  &  &  \\
SD-Transformer & BPTT & LIF & SDSA &  &  &  \\
Spikingformer & BPTT & LIF & SSA &  &  &  \\
Auto-Spikformer & BPTT & LIF & SSA & $\checkmark$ &  &  \\
SPIKEBERT & BPTT & LIF & SSA &  &  & $\checkmark$ \\
SpikingBERT & Implicit Diff & LIF & ASR-SA &  &  & $\checkmark$ \\
SpikeGPT & BPTT & LIF & S-RWKV &  & $\checkmark$ &  \\
SpikformerV2 & BPTT & LIF & SSA &  & $\checkmark$ &  \\
MST & Conversion & IF & VSA &  & $\checkmark$ &  \\
\bottomrule
\end{tabular}
}

\end{table}

\paragraph{Transformer-based SNNs.}
The Transformer-based ANN \cite{vaswani2017attention,dosovitskiy2020image} comprises three key components: 1) an embedding layer designed to convert image patches or words in sentences to tokens for enhanced learning; 2) cascaded Transformer encoders, incorporating several self-attention and multi-layer perceptron blocks, aimed at learning spatial-temporal features within the tokens; 3) shallow head responsible for executing specific tasks. 
Notably, the Transformer-based ANNs attain state-of-the-art accuracy in both the CV and NLP tasks, thereby catalyzing the advancement of Transformer-based SNNs, as tabulated in \cref{related_work}. 

Transformer-based SNNs are initially pursued through DT methods, such as Spikformer v1/v2 \cite{zhou2022spikformer,zhou2024spikformer}, SD-Transformer \cite{yao2023spike} and SpikeGPT \cite{zhu2023spikegpt}. 
Spikingformer \cite{zhou2023spikingformer} tackles \textit{non-spike computation} challenges by swapping the positions of convolution and batch normalization. 
To further improve the accuracy of the SNN, SPIKEBERT \cite{lv2023spikebert} and SpikingBERT \cite{bal2023spikingbert} employ knowledge distillation algorithms to transfer information from ANN.
Furthermore, SpikeGPT \cite{zhu2023spikegpt} and spikformer V2 \cite{zhou2024spikformer} leverage the pre-train algorithm for better-initiated model weights, while Auto-Spikformer \cite{che2023auto} leverages the network architecture search (NAS) to identify a spiking Transformer with high accuracy and low energy consumption. 
Besides, SpikingBERT \cite{bal2023spikingbert} adopts an implicit differentiation algorithm, distinct from previous works using BPTT, for SNN training. 
In contrast, for A2S methods, MST \cite{wang2023masked} replaces the QCFS \cite{bu2023optimal} activation function to ReLU function in ANN and converts the Transformer-based ANN to SNN.

\begin{figure*}[t]
\centering
\includegraphics[width=\linewidth]{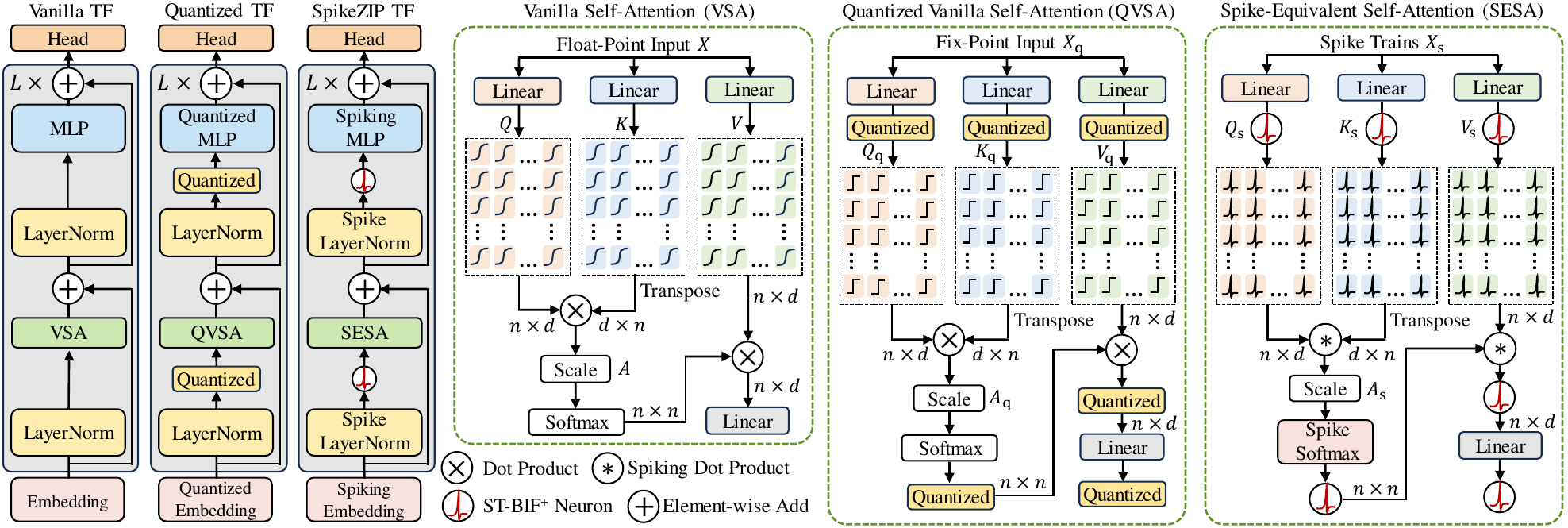}
\caption{\textbf{Architecture of Transformer-based SNN in SpikeZIP-TF.} Compared to the vanilla Transformer, SpikeZIP-TF inserts the ST-BIF\textsuperscript{+} neuron ahead of and behind the matrix multiplication operations and substitutes SNN-unfriendly operators (dot product, Softmax and LayerNorm) with SNN-friendly ones (spiking dot product, Spike-Softmax and Spike-LayerNorm). TF: Transformer; $n$: sequence length; $d$: token dimension; $\{\mQ,\mK,\mV,\mA\}$, $\{\mQ_{\textrm{q}},\mK_{\textrm{q}},\mV_{\textrm{q}},\mA_{\textrm{q}}\}$, $\{\mQ_{\textrm{s}},\mK_{\textrm{s}},\mV_{\textrm{s}},\mA_{\textrm{s}}\}$: \{query, key, value, attention array\}, \{theirs quantized form\} and \{spike form\}.}
\label{fig:spikezip_arch}
\end{figure*}

\emph{The limited adoption of A2S methods in Transformer-based SNNs stems from the challenge of establishing mathematical equivalence between operators in quantized Transformer-based ANNs and SNNs.} Existing A2S methods have yet to tackle the equivalence issues associated with the following operators: \emph{self-attention, softmax, and layer normalization} \cite{ba2016layer}.
In SpikeZIP-TF, we address the operator equivalence challenge by introducing a novel spiking equivalence self-attention (\textit{aka}. SESA).
Additionally, for softmax and layer-norm, we employ a differential algorithm to design their equivalent spiking forms.
By integrating our spiking operators, SpikeZIP-TF establishes equivalence between quantized Transformer-based ANNs and SNNs.

\section{Methods}

\subsection{Dynamics of ST-BIF\textsuperscript{+} Neuron}
\par In SpikeZIP-TF, to address the inequivalence between ST-BIF neuron and quantized function in Transformer, we propose the ST-BIF\textsuperscript{+} whose dynamics can be expressed as:
\begin{equation}
\label{eqt:ST-BIF}
\begin{gathered}
V_t = 
V_{t-1} + V^\textrm{in}_t - V_\textrm{thr} \cdot \Theta(V_{t-1} + V^\textrm{in}_t, V_\textrm{thr}, S_{t-1}) \\
S_t = S_{t-1} + \Theta(V_{t-1} + V^\textrm{in}_t, V_\textrm{thr}, S_{t-1}) \\
\Theta(V, V_\textrm{thr}, S) = 
\begin{cases}
1 ;& V \geq V_\textrm{thr} ~ \& ~ S < S_{\textrm{max}} \\
0 ;&  \textrm{other} \\
-1 ;&  V < 0 ~ \& ~ S > S_{\textrm{min}}
\end{cases}
\end{gathered}
\end{equation}
where the notations are specified in \cref{tab:notation_table}. Compared to the ST-BIF neuron, the ST-BIF\textsuperscript{+} neuron expands the minimum value of the spike tracer from zero to the lower clamp bound in the quantized function as follows: 
\begin{equation}
\label{eqt:qrelu_mp}
{\rm Quantize}(x) = s \cdot \textrm{clamp}(\textrm{round}(x/s),\alpha,\beta)    
\end{equation}
where $s$ represents the quantization scale size, $\alpha, \beta$ represent the minimum and maximum of clamp range in the quantizer. By setting $V_{\rm thr}=s, S_{\rm min} = \alpha, S_{\rm max} = \beta$, the accumulated output of ST-BIF\textsuperscript{+} is equivalence to the quantized function.

\begin{table}[t]
\centering
\caption{Summary of mathematical notations used in this paper.}
\label{tab:notation_table}
\resizebox{\linewidth}{!}{
\renewcommand{\arraystretch}{1.0}
\begin{tabular}{cl}
\toprule
Notation & \multicolumn{1}{c}{Description} \\ 
\midrule
$V_t$ & \begin{tabular}[c]{@{}l@{}}potential of neuron membrane at time-step $t$\end{tabular} \\
$V_{\textrm{thr}}$ & \begin{tabular}[c]{@{}l@{}}threshold voltage for neuron to fire a spike\end{tabular} \\
$V^{\textrm{in}},V^{\textrm{out}}$ & \begin{tabular}[c]{@{}l@{}}input or output voltage of neuron \end{tabular} \\
$S_{t}$ & \begin{tabular}[c]{@{}l@{}}spike tracer at time-step $t$\end{tabular} \\
$S_{\textrm{max}}, S_{\textrm{min}}$ & \begin{tabular}[c]{@{}l@{}}maximum and minimum bound of spike tracer \end{tabular} \\
$\textrm{clip}(x,\alpha_\textrm{min},\alpha_\textrm{max})$ & \begin{tabular}[c]{@{}l@{}}clip function that limits $x$ within  $\alpha_\textrm{min}$ and $\alpha_\textrm{max}$\end{tabular} \\
$\Theta(V,V_{\textrm{thr}},S)$ & output spike decision function of ST-BIF\textsuperscript{+} \\
$T_{\textrm{eq}}$ & time-step of SNN enters the equilibrium state \\
\bottomrule
\end{tabular}
}
\end{table}

\begin{figure*}[t]
\centering  
\subfigbottomskip=2pt 
\subfigcapskip=-5pt 
\subfigure[Activation-Weight Multiplication]{
    \includegraphics[width=0.28\linewidth]{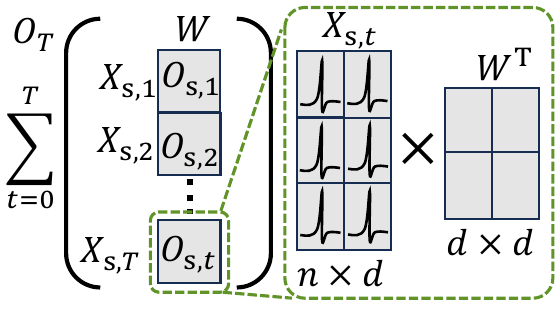}}
\subfigure[Activation-Activation Multiplication]{
    \includegraphics[width=0.68\linewidth]{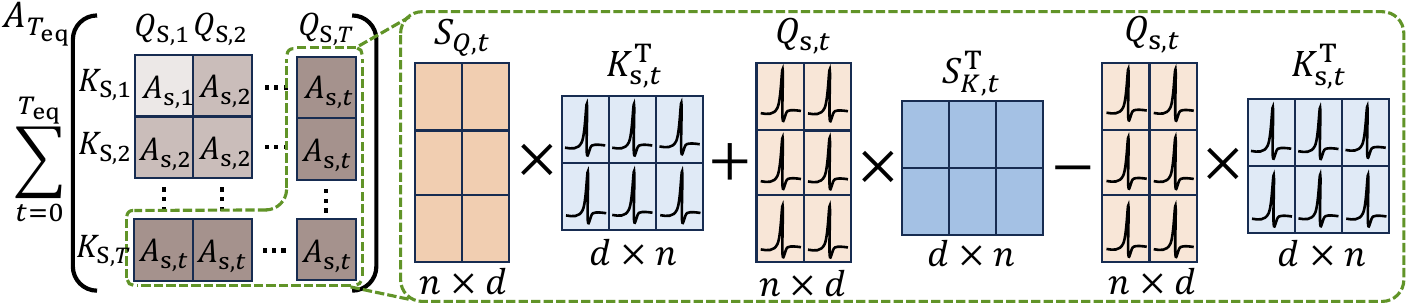}}
\caption{\textbf{The process of matrix multiplication in SESA.} The bracket part in (a) and (b) corresponds to \cref{eqt:multi1} and  \cref{eqt:multi2} respectively. 
(a) $\mX_{\textrm{s},t}, \mO_{\textrm{s},t}$ represent the input and output spike trains in SNN at time-step $t$. 
(b) $\mQ_{\textrm{s},t}, \mK_{\textrm{s},t}$ denote the query and key in SNN at time-step $t$; $\mS_{\textrm{Q},t}, \mS_{\textrm{K},t}$ are the spike tracers in the neuron layers, which store the accumulated output for query and key. At each time-step, we utilize the accumulated output in the spike tracer to perform AA multiplication via three matrix multiplications. }
 \label{fig:matrix_precess}
\end{figure*}

\subsection{Transformer-based SNN in SpikeZIP-TF}
In this section, we first elaborate on the network topology of Transformer-based SNN in SpikeZIP-TF. 
Then, for equivalent ANN-SNN conversion, we introduce SNN-friendly operators in SpikeZIP-TF including SESA, Spike-Softmax and Spike-Layernorm.

\subsubsection{Architecture Overview}

As shown in \cref{fig:spikezip_arch}, the architecture of Transformer-based SNN generated by SpikeZIP-TF is almost identical to the vanilla Transformer, which consists of an embedded layer, a Transformer encoder, and a shallow head. 
Given a target Transformer-based ANN and to obtain its SNN counterpart, we coduct the following procedures:
\begin{enumerate}
    \item Quantizers are inserted both ahead of and behind the matrix multiplication operators to acquire a quantized Transformer, which is consistent with prior quantized Transformer work, \textit{e.g.}, I-BERT \cite{kim2021bert};
   
    \item Once the model with quantization function is trained with QAT, quantized functions are replaced with ST-BIF\textsuperscript{+} neuron to ensure the inputs and outputs of matrix multiplication are in the form of spike trains;
    
    \item SNN-unfriendly operators in the quantized Transformer (\textit{e.g.}, Softmax, LayerNorm and dot product) are subsituted with SNN-friendly operators (\textit{e.g.}, Spike-Softmax, Spike-LayerNorm and spiking dot product).
\end{enumerate}

\subsubsection{Embedding \& Head for SNN}
Compared to the embedding layer in QANN, an ST-BIF\textsuperscript{+} neuron layer is introduced following the embedding layer to facilitate the conversion of analog input charge into spike trains for Transformer encoding.
For the head of SNN, following the previous SNN works \cite{wang2023masked,zhou2024spikformer,yao2023spike}, we use membrane potential in the ST-BIF$^+$ rather than spike trains as the output.

\subsubsection{Spike-Equivalent Self-Attention (SESA)}

As a pivotal component within the SpikeZIP-TF, the design of spike-equivalent self-attention (SESA) adheres to two fundamental principles: 1) ensuring that the accumulated output remains \emph{equivalent to quantized vanilla self-attention}, and 2) \emph{aligning with the computing paradigm in SNN}.

SESA is implemented as described in \cref{fig:spikezip_arch}. 
There are two types of activation matrix multiplication in SESA, \textit{i.e.}, 1) \textit{Activation-Weight} (\textit{aka}. AW) multiplication in the linear layer, where one operand is stationary while the other is dynamically generated; 2) \textit{Activation-Activation} (\textit{aka}. AA) multiplication, where both operands are generated on-the-fly. It occurs between the query $\mQ$ and key $\mK$, as well as attention array $\mA= \mQ\mK$ and value $\mV$. The AA multiplication is also presented as the spiking dot product in \cref{fig:spikezip_arch}, \textit{i.e.}, both tensors are composed by spikes.

\paragraph{AW Multiplication.} 
Thanks to the equivalence between ST-BIF\textsuperscript{+} and the quantized function, we can conclude that:
\begin{equation}
\begin{aligned}
    \mO_{T_{\textrm{eq}}} &= \mW \cdot \mX_{\textrm{q}} = \sum_{t=0}^{T_{\textrm{eq}}} (\mW \cdot  \mX_{\textrm{s},t}); x_{\textrm{s},t} \in \{0,\pm1\} 
    \label{eqt:multi1}
\end{aligned}
\end{equation}
where $\mW$ denotes the weight of linear layer, while other notations are summarized in \cref{tab:notation_table} and \cref{fig:spikezip_arch}.
Note that, when $t > T_{\textrm{eq}}$, all neurons enter the \emph{equilibirum state}, during which they neither receive nor fire any spike. 

\paragraph{AA Multiplication.}
Taking the multiplication between query and key as an example, it can be written as:
\begin{equation}
\begin{aligned}
\label{eqt:multi2}
\mA_{T_{\textrm{eq}}} &= \mQ_{\textrm{q}} \cdot \mK_{\textrm{q}} = \sum_{t_1=0}^{T_{\textrm{eq}}} \mQ_{\textrm{s},t_1} \cdot \sum_{t_2=0}^{T_{\textrm{eq}}} \mK_{\textrm{s},t_2} \\
&= \sum_{t=0}^{T_{\rm eq}} \mS_{{\rm Q},t} \cdot \mK_{{\rm s},t}^T +  \mQ_{{\rm s},t} \cdot \mS_{{\rm K},t}^T - \mQ_{{\rm s},t} \cdot \mK_{{\rm s},t}^T
\end{aligned}
\end{equation}
where $\mQ_{\textrm{s},t_1}, \mK_{\textrm{s},t_2}$ represent the spike trains of query and key at time-step $t_1$ and $t_2$ respectively, $\mS_{{\rm Q},t}, \mS_{{\rm K},t}$ represent the accumulated spike trains of query and key which are stored in ST-BIF\textsuperscript{+} neuron, while $\mA_{T_{\textrm{eq}}}$ denotes the attention array accumulated during $T_{\textrm{eq}}$.

To compute \cref{eqt:multi1,eqt:multi2} 
in SESA, the matrix multiplication must be decomposed into sub-operations for each time-step. 
To address this, we propose a novel calculation method for matrix multiplication. 
The detailed processes of AW and AA multiplication are outlined in \cref{fig:matrix_precess}. 
For AW multiplication, as depicted in \cref{fig:matrix_precess}(a), 
we perform matrix multiplication on the input spike trains per time-step to generate the output.
In contrast, for AA multiplication, we utilize the accumulated output in the spike tracer to compute the attention array through three matrix multiplications, as illustrated in \cref{fig:matrix_precess}(b). 
With the computation depicted in \cref{fig:matrix_precess}, we achieve a lossless conversion from QVSA to SESA while adhering to the computing paradigm in SNN.

\subsubsection{Spike-Softmax \& Spike-LayerNorm}

To enable Softmax and LayerNorm operations in SNN, we introduce Spike-Softmax and Spike-LayerNorm, which are equivalent to their ANN counterparts. 
The process of Spike-Softmax and Spike-LayerNorm can be expressed as: 
\begin{equation}
\begin{gathered}
\label{eqt:spiking_version}
\quad \mX_{T} = { \textstyle \sum_{t=0}^T} \mX_{\textrm{s},t}; \quad
\mO_{T} = {\rm \sigma}(\mX_{T}) \\
\mO_{\textrm{s},t} = \mO_{t} - \mO_{t-1} 
\end{gathered}
\end{equation}
where $\sigma$ represents the function of Softmax or LayerNorm. $\mX_{\textrm{s},t}$ and $\mO_{\textrm{s},t}$ are the input and output of the operator at time-step $t$ respectively, $\mX_T$ is the summation of the input during $T$ time-steps, which is stored in the operator. $\mO_T$ is the output of the function $\sigma$ with input $\mX_T$. 
The Spike-Softmax(LayerNorm) can be made equivalent to Softmax(LayerNorm) by summing up $\mO_{\textrm{s},t}$ through time. 

\subsection{Complexity Analysis}

\begin{table}[t]
\caption{\textbf{Complexity analysis of SpikeZIP-TF} for its operators of AW(AA)-Multiplication (\textit{abbr}. Mult), Spike-Softmax (SSoftmax), and Spike-LayerNorm (SLayerNorm). $n$: sequence length; $d$: token demension; $T$: \# time-steps; $\gamma$: Performance ratio of one operation in ANN versus SNN.
}
\label{tab:complex}
\resizebox{\linewidth}{!}{
\begin{tabular}{@{}rccccc@{}}
\toprule
 & \multicolumn{1}{l}{Network} & AW-Mult  & AA-Mult  & SSoftmax & SLayerNorm \\ \midrule
Spatial & \multirow{2}{*}{SNN} & $O(nd+d^2)$ & $O(nd)$ & $O(n^2)$ & $O(nd)$ \\
Temporal &  & $O(Tnd^2)$ & $O(Tnd^2)$ & $O(Tn^2)$ & $O(Tnd)$ \\ \midrule
 & \multicolumn{1}{l}{Network} & AW-Mult & AA-Mult & Softmax & LayerNorm \\ \midrule
Spatial & \multirow{2}{*}{(Q)ANN} & $O(nd+d^2)$ & $O(nd)$ & $O(1)$ & $O(1)$ \\
Temporal &  & $O(\gamma nd^2)$ & $O(\gamma nd^2)$ & $O(\gamma n^2)$ & $O(\gamma nd)$ \\ \bottomrule
\end{tabular}}

\end{table}

The spatial and temporal complexity analysis of the operations in SpikeZIP-TF is presented in \cref{tab:complex}. 
Note that, the synaptic operations in SNN are addition or subtraction for AW multiplication and binary operation for AA multiplication, while ANN is integer or floating-point multiplication \cite{horowitz20141}.
Therefore, we introduce an ratio of $\gamma >> 1$ to indicate the higher operation cost of (Q)ANN \textit{w.r.t} SNN. 
For spatial complexity, AW multiplication and AA multiplication do not bring additional cost, but Spike-Softmax and Spike-LayerNorm have $n^2$ and $nd$ times more complexity than ANN. 
It is resulted from that Spike-Softmax and Spike-LayerNorm requires extra memory to store the accmulated input.

\section{Experiments}

\subsection{Experimental Setup}

\paragraph{Vision Benchmarks.}
Various vision datasets are adopted for evaluation.
\textbf{1) static vision datasets}, including CIFAR10/100 \cite{krizhevsky2009learning} and ImageNet \cite{deng2009imagenet}.
\textbf{2) neuromorphic vision dataset}: We evaluate SpikeZIP-TF on CIFAR10-DVS~\cite{2017CIFAR10}. CIFAR10-DVS is a neuromorphic event-stream dataset with 10 distinct classes, which is created by leveraging the dynamic vision sensor~(DVS) to convert 10k frame-based images from CIFAR10 \cite{krizhevsky2009learning} dataset into 10k event streams.
For ImageNet, we apply the pre-trained Vision Transformer-Small/Base/Large (\textit{aka}. ViT-S/B/L) \cite{dosovitskiy2020image} as the source ANN. For CIFAR-10/100 and CIFAR10-DVS, we utilize the pre-trained Vision Transformer-Small (\textit{aka}. ViT-S) as the source ANN.

\paragraph{NLP Benchmarks.}
Various natural language understanding (NLU) datasets are evaluated, including English (MR \cite{pang2004sentimental}, Subj \cite{pang2004sentimental}, SST-2, SST-5 \cite{socher2013recursive}) and Chinese (ChnSenti, Waimai). 
For NLP tasks, the Roberta-Base/Large \cite{DBLP:journals/corr/abs-1907-11692} (\textit{aka}. Roberta-B/L) is chosen as source ANN owing to its high accuracy in NLP benchmarks.

\subsection{Results Comparison}

\begin{table}[h]
\centering
\caption{\textbf{Experimental results on CIFAR-10, CIFAR-100 and CIFAR10-DVS.} \textit{CF} is the abbreviation of CIFAR. The best results are in \textbf{bold}, the runner-up results are in \colorbox{baselinecolor}{gray}.}
\label{tab:cifar}
\resizebox{\linewidth}{!}{
\begin{tabular}{@{}ccccccccc@{}}
\toprule
\multirow{2}{*}{Methods} & \multirow{2}{*}{Category} & \multirow{2}{*}{Param(M)} & \multicolumn{2}{c}{\textit{CF}10-DVS} & \multicolumn{2}{c}{\textit{CF}-10} & \multicolumn{2}{c}{\textit{CF}-100} \\ 
 & & & Acc & T & Acc & T & Acc & T \\ \midrule
 ViT-S  & ANN     &   21.70   & 90.4  & 1     & 99.2 & 1 & 91.9 & 1 \\ \midrule
 QViT-S-8Level  & \multirow{3}{*}{QANN}      &   \multirow{3}{*}{21.70}   & -  & -     & 98.0 & 1 & 87.2 & 1 \\
 QViT-S-16Level  &       &   & 88.4  & 1     & 98.7 & 1 & 89.5 & 1 \\
  QViT-S-32Level  &       &   & 90.2  & 1     & - & - & - & - \\ \midrule
tdBN & \multirow{5}{*}{\begin{tabular}[c]{@{}c@{}}SNN \\ (Direct\\ Training)\end{tabular}} & / & 67.8 & 10 & 93.2 & 6 & - & - \\
ASpikformer & & 8.46 & / & / & 96.4 & 4 & 78.2 & 4 \\
Spikformer) & & 9.32 & 80.9 & 16 & 95.5 & 4 & 78.2 & 4 \\
SDformer & & / & 80.0 & 16 & 95.6 & 4 & 78.4 & 4 \\
Sformer+CML & & 9.32 & 81.4 & 16 & 96.0 & 4 & 80.4 & 4 \\ \midrule
MST & \multirow{3}{*}{\begin{tabular}[c]{@{}c@{}}SNN\\(A2S)\end{tabular}} & 27.60 & \baseline{88.1} & 512 & 97.3 & 256 & 86.9 & 256 \\
\multirow{2}{*}{\begin{tabular}[c]{@{}c@{}}SpikeZIP-TF\\(ours)\end{tabular}} & & \multirow{2}{*}{21.70} & 87.6 & 32 & \baseline{97.7} & 16 & \baseline{87.3} & 16 \\ 
 & & & \textbf{90.5} & 64 & \textbf{98.7} & 32 & \textbf{89.7} & 32 \\ 
\bottomrule
\end{tabular}
}
\end{table}

\begin{table*}[t]
\centering
\caption{\textbf{Comparison on ImageNet.} $\diamondsuit$: For ViT-S, we finetune the pretrained model from AugReg \cite{DBLP:journals/corr/abs-2106-10270} with ReLU activation to achieve corresponding ANN. $\dagger$: For ViT-B/L, we finetune the pretrained models from MAE \cite{He_2022_CVPR} with ReLU activation to achieve corresponding ANN. $\star$: LSQ \cite{DBLP:journals/corr/abs-1902-08153} quantization results of corresponding MAE pre-trained models. Our SpikeZIP-TF results are the equivalent conversion from corresponding quantization results. The prefix \textbf{Level} in architecture column means the quantization level. The best results are in \textbf{bold}, and the runner-up results are in \colorbox{baselinecolor}{gray}. }
\label{tab:imagenet}
\resizebox{.75\linewidth}{!}{
\begin{tabular}{cccccc}
\toprule
Category  & Methods             & Architecture       & Param(M) & Time-Step & Acc(\%) \\ \midrule
\multirow{6}{*}{ANN}            & T2T-ViT             & T2T-ViT-24        & 64.10  & 1     & 82.30  \\
 & PVT  & PVT-Large         & 61.40  & 1     & 81.70  \\
 & Swin Transformer        & SWIN-Base         & 88.00  & 1     & 83.50  \\
 & AugReg$\diamondsuit$  & ViT-S         & 22.05  & 1     & 82.34  \\
 & \multirow{2}{*}{MAE~(ReLU)$\dagger$} & ViT-B        & 86.57  & 1     & 83.75  \\
 & & ViT-L         & 304.33  & 1     & 85.41  \\ \midrule
\multirow{3}{*}{QANN}            & \multirow{3}{*}{LSQ$\star$}      & QViT-S-32Level      & 22.05  & 1     & 81.59  \\
 & & QViT-B-32Level       & 86.57  & 1     & 82.83  \\
 & & QViT-L-32Level      & 304.33  &  1    & 83.86  \\ \midrule
\multirow{16}{*}{\begin{tabular}[c]{@{}c@{}} SNN \\ (Direct\\ Training)\end{tabular}} & \multirow{2}{*}{TET}      & Spiking-ResNet-34     & 21.79  & 6     & 64.79  \\
 & & SEW-ResNet-34       & 21.79  & 4     & 68.00  \\
 & STBP-tdBN            & Spiking-ResNet-34     & 21.79  & 6     & 63.72  \\
 & \multirow{4}{*}{SEW ResNet}   & SEW-ResNet-34       & 21.79  & 4     & 67.04  \\
 & & SEW-ResNet-50       & 25.56  & 4     & 67.78  \\
 & & SEW-ResNet-101      & 44.55  & 4     & 68.76  \\
 & & SEW-ResNet-152      & 60.19  & 4     & 69.26  \\
 & Attention-SNN          & ResNet-104        & 78.37  & 4     & 77.08  \\
 & Spike-driven Transfromer    & Spiking Transformer-8-768 & 66.34  & 4     & 77.07  \\
 & Spikingformer          & Spiking Transformer-8-768 & 66.34  & 4     & 75.85  \\
 & CML  & Spiking Transformer-8-768 & 66.34  & 4     & 77.64  \\
 & \multirow{2}{*}{Spikformer}   & Spikformer-6-512     & 23.37  & 4     & 77.26  \\
 & & Spikformer-8-768     & 66.34  & 4     & 79.55  \\
 & \multirow{3}{*}{Spikformer V2} & Spikformer V2-8-384    & 29.11  & 4     & 78.80  \\
 & & Spikformer V2-8-512    & 51.55  & 4     & 80.38  \\
 & & Spikformer V2-16-768   & 172.70  & 4     & 82.35  \\ \midrule
\multirow{10}{*}{ \begin{tabular}[c]{@{}c@{}}SNN\\(A2S)\end{tabular} }        & Hybrid training        & ResNet-34         & 21.79  & 250    & 61.48  \\
 & \multirow{2}{*}{Spiking ResNet} & ResNet-34         & 21.79  & 350    & 71.61  \\
 & & ResNet-50         & 25.56  & 350    & 72.75  \\
 & QCFS & VGG-16          & 138.42  & 64    & 72.85  \\
 & Fast-SNN            & VGG-16          & 138.42  & 7     & 72.95  \\
 & COS  & ResNet-34         & 21.79  & 8     & 74.17  \\
 & MST  & Swin-T (BN)        & 28.5   & 512    & 78.51  \\
 & \multirow{3}{*}{\begin{tabular}[c]{@{}c@{}}SpikeZIP-TF\\(ours)\end{tabular}}      & SViT-S-32Level      & 22.05  & 64    & 81.45  \\
 & & SViT-B-32Level       & 86.57  & 64    & \baseline{82.71}  \\
 & & SViT-L-32Level      & 304.33  & 64    & \textbf{83.82}  \\ \bottomrule
\end{tabular}
}
\end{table*}

\paragraph{Comparison on CIFAR-10/100} of SpikeZIP-TF and previous methods are elaborated in~\cref{tab:cifar}, revealing SpikeZIP-TF's superiority over prior approaches across both CIFAR-10 and CIFAR-100 datasets. 
Notably, with ViT-S as the backbone, SpikeZIP-TF surpasses MST \cite{wang2023masked} by 1.4\% on CIFAR-10 and 2.8\% on CIFAR-100 with 8$\times$ less time-steps.
Compared with direct training methods, SpikeZIP-TF exhibits a 2.7\% and 9.3\% improvement over Spikingformer+CML on CIFAR-10 and CIFAR-100, respectively.

\paragraph{Comparison on CIFAR10-DVS} is reported in \cref{tab:cifar} as well. To expedite the training convergence via leveraging the pre-trained weights, we adopt the pre-processing approach outlined in \cite{wang2023masked}. 
This involves adding an additional reduction layer to reduce the channel dimension of neuromorphic data to 3. The experimental results in \cref{tab:cifar} underscore the effectiveness of SpikeZIP-TF in processing neuromorphic datasets. 
Compared to A2S-based MST~\cite{wang2023masked}, SpikeZIP-TF achieves a 2.4\% higher accuracy with fewer time-steps.
Despite the lower time-step requirement of the previous SOTA direct training method~(Spikingformer+CML), SpikeZIP-TF delivers a remarkable 9.1\% accuracy boost on CIFAR10-DVS. 

\begin{table*}[t]
\centering
\caption{\textbf{Comparison on NLU datasets.} The source ANN of SpikeZIP-TF with 125M param and 355M param are Roberta-B and Roberta-L. Cat. is short for Catogery. $\ddag$: results taken from the SpikeBERT \cite{lv2023spikebert}. 
$\dagger$: Results of Roberta with ReLU activation. $\star$: LSQ quantization results of corresponding Roberta pre-trained models. The best results are in \textbf{bold}, the runner-up results are in \colorbox{baselinecolor}{gray}. }
\label{tab:NLU}
\resizebox{.9\linewidth}{!}{
\begin{tabular}{@{}ccccccclccclc@{}}
\toprule
\multirow{2}{*}{Methods} & \multirow{2}{*}{\begin{tabular}[c]{@{}c@{}}Param\\ (M)\end{tabular}} & \multirow{2}{*}{Cat.} & \multicolumn{5}{c}{English Dataset} & \multicolumn{1}{c}{\multirow{2}{*}{Avg.}} & \multicolumn{3}{c}{Chinese Dataset} & \multirow{2}{*}{Avg.} \\ \cmidrule(lr){4-8} \cmidrule(lr){10-12}
 & & & MR & SST-2 & Subj & SST-5 & T. & \multicolumn{1}{c}{} & ChnSenti & Waimai & T. & \\ \midrule
TextCNN & n/a & \multirow{3}{*}{ANN} & 77.41 & 83.25 & 94.00 & 45.48 & 1 & 75.04 & 86.74 & 88.49 & 1 & 87.62 \\
Roberta-B$\dagger$ & 125 & & 87.16 & 94.15 & 96.30 & 54.57 & 1 & 83.05 & 88.22 & 92.05 & 1 & 90.14 \\ 
Roberta-L$\dagger$ & 355 & & 91.33 & 96.21 & 97.25 & 57.42 & 1 & 85.55 & 86.90 & 92.91 & 1 & 89.91 \\ \midrule
Roberta-B-32Level$\star$ & 125 & \multirow{2}{*}{QANN} & 85.76 & 92.81 & 95.55 & 52.71 & 1 & 81.71 & 88.36 & 91.88 & 1 & 90.12 \\
Roberta-L-64Level$\star$ & 355 & & 88.77 & 93.24 & 96.70 & 56.11 & 1 & 83.71 & 87.03 & 91.80 & 1 & 89.42 \\ \midrule
SNN-TextCNN & - & \multirow{5}{*}{\begin{tabular}[c]{@{}c@{}}Direct\\ Training\end{tabular}} & 75.45 & 80.91 & 90.60 & 41.63 & 50 & 72.15 & 85.02 & 86.66 & 40 & 85.84 \\
Spikformer$\ddag$ & 110 & & 76.38 & 81.55 & 91.80 & 42.02 & 4 & 72.94 & 85.45 & 86.93 & 4 & 86.19 \\
SpikeBERT & 109 & & 80.69 & 85.39 & 93.00 & 46.11 & 4 & 76.30 & 86.36 & 89.66 & 4 & 88.01 \\
\multirow{2}{*}{SpikeGPT} & 45 & & 69.23 & 80.39 & 88.45 & 37.69 & 50 & 68.94 & n/a & n/a & n/a & n/a \\
 & 216 & & 85.63 & 88.76 & 95.30 & 51.27 & 50 & 80.24 & n/a & n/a & n/a & n/a \\ \midrule
\multirow{2}{*}{\begin{tabular}[c]{@{}c@{}}SpikeZIP-TF\\ (ours)\end{tabular}} & 125 & \multirow{2}{*}{A2S} & \baseline{86.13} & \baseline{92.81} & \baseline{95.55} & \baseline{52.71} & 64 & 
\baseline{81.80} & \baseline{86.77} & \textbf{91.88} & 64 & \textbf{89.33} \\
 & 355 & & \textbf{89.28} & \textbf{93.79} & \textbf{96.70} & \textbf{56.51} & 128 & \textbf{84.07} & \textbf{87.16} & \baseline{91.29} & 128 & \baseline{89.23} \\ \bottomrule
\end{tabular}}
\end{table*}

\paragraph{Comparison on ImageNet} of SpikeZIP-TF and previous methods is tabulated in~\cref{tab:imagenet}. As anticipated, SpikeZIP-TF surpasses previous SOTA methods.
Compared to the SOTA A2S conversion method (MST \cite{wang2023masked}), SpikeZIP-TF achieves 2.94\% higher top-1 accuracy while utilizing fewer time-steps and a more lightweight model (with 6.4M parameter reduction). 
Although direct training methods such as Spikformer \cite{zhou2022spikformer}, Spikingformer \cite{zhou2023spikingformer} and Spikformer V2 \cite{zhou2024spikformer} require lower time-step, they demand relatively high training cost to achieve compatible performance with ANN-to-SNN conversion-based methods.
In contrast, compared to previous SOTA on direct training methods~(Spikformer V2~\cite{zhou2024spikformer}), SpikeZIP-TF incurs significantly lower computational cost while maintaining a more lightweight model and achieving SOTA top-1 accuracy. 
For large-scale models~(ViT-L), after simply fine-tuning and quantizing the publicly available pre-trained ANN, SpikeZIP-TF yields promising performance~(83.28\% on ImageNet).

\paragraph{Comparison on NLP Benchmarks}
We conduct a comparative analysis of SpikeZIP-TF with other SOTA works, including SNN-TextCNN \cite{lv2022spiking}, SpikeGPT \cite{zhu2023spikegpt} and SpikeBERT \cite{lv2023spikebert}. 
The results are summarized in \cref{tab:NLU}. 
SpikeZIP-TF outperforms SpikeGPT \cite{zhu2023spikegpt} and SpikeBERT \cite{lv2023spikebert} in terms of accuracy across both the English datasets and Chinese datasets. 
The improvements in accuracy are particularly notable in the MR (3.65\% increase) and SST-5 (5.24\% increase) datasets. 
Moreover, SpikeZIP-TF achieves the highest accuracy despite having a greater model size (355M) compared to SpikeGPT (216 M). 
It is noteworthy that, as shown in \cref{tab:NLU}, SpikeZIP-TF converted from Roberta-L exhibits lower performance compared to the SpikeZIP-TF converted from Roberta-B. 
This difference can be attributed to the relatively lower accuracy of the pre-trained Roberta-L model compared to the pre-trained Roberta-B model.

\begin{table}[t]
\caption{\textbf{Comparison of training cost.} The SpikeZIP-TF consumes fewer training hours and less energy than SpikeGPT and Spikformer V2. \textbf{Acc.}: short for accuracy. \textbf{N.}: short for Nvidia}
\label{tab:train_cost}
\resizebox{\linewidth}{!}{
\begin{tabular}{@{}ccccccc@{}}
\toprule
Method & Params & Dataset & GPU & Time(h) & Energy(kw/h) & Acc.(\%) \\ \midrule
SpikeGPT & 216 & \multirow{2}{*}{SST-2} & 4 N & 48.0 & 57.6 & 88.76 \\
SpikeZIP-TF & 355 &  & 1 N & \textbf{1.03} & 0.36 & \textbf{93.79} \\ \midrule
SpikformerV2 & 172.7 & \multirow{2}{*}{\begin{tabular}[c]{@{}c@{}}Image\\ Net\end{tabular}} & 8 N & 196.7 & 472.08 & 81.10 \\
SpekeZIP-TF & 304.3 &  & 8 N & \textbf{30.0} & 108.00 & \textbf{83.82} \\ \bottomrule
\end{tabular}}
\end{table}

\subsection{Training Cost Analysis}
One of the key advantages of SpikeZIP-IF is its low training cost, as illustrated in \cref{tab:train_cost}.
SpikeZIP-TF exhibits lower training hours and consumes less energy compared to SpikeGPT and Spikformer V2.
This efficiency stems from SpikeZIP-TF's ability to skip the pre-training stage by leveraging the pre-trained ANN accessed accessed in open sources (\textit{e.g.}, Pytorch Hub, huggingface, etc.) to initialize the quantization-aware training in QANN fine-tuning.

\subsection{Power Estimation on Neuromorphic Hardware}

\begin{table}[t]
\caption{\textbf{The power consumption of SpikeZIP-TF and other works.} The time-step of SpikeGPT used in power estimation is larger than 50, therefore the power is less than 0.234. }
\label{tab:power}
\resizebox{\linewidth}{!}{
\begin{tabular}{@{}cccccc@{}}
\toprule
Method & Params(M) & Dataset & Time-Steps & Power(W) & Acc(\%) \\ \midrule
SpikeBERT & 109 & \multirow{2}{*}{SST-2} & 4 & 7.135 & 85.39 \\
SpikeZIP-TF & 355 &  & 64 & 4.320 & 93.79 \\ \midrule
Spikformer & 66.34 & \multirow{6}{*}{ImageNet} & 4 & 8.02 & 74.81 \\
Spikingformer & 66.34 &  & 4 & 3.42 & 75.85 \\
Spikformer~V2 & 64.18 &  & 4 & 3.67 & 81.17 \\
Spikformer~V2 & 172.70 &  & 4 & 6.39 & 82.35 \\
SpikeZIP-TF & 86.57 &  & 64 & 6.30 & 82.71 \\
SpikeZIP-TF & 304.33 &  & 64 & 19.85 & 83.82 \\ \bottomrule
\end{tabular}}
\end{table}

\begin{figure*}[t]
    \centering
    \subfigure[NLU datasets]{
        \includegraphics[width=0.23\textwidth]{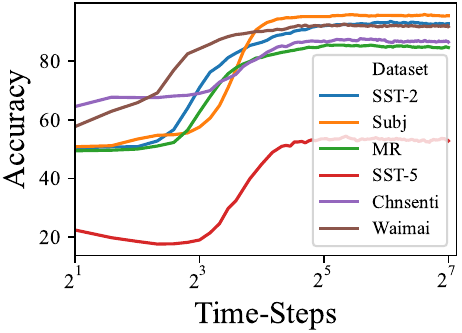}
        \label{fig:subfig1}
    }
    \subfigure[CV datasets]{
        \includegraphics[width=0.23\textwidth]{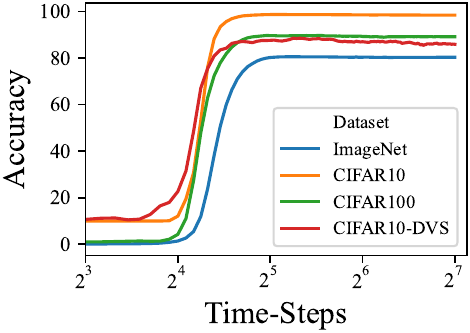}
        \label{fig:subfig2}
    }
    \subfigure[model size]{
        \includegraphics[width=0.23\textwidth]{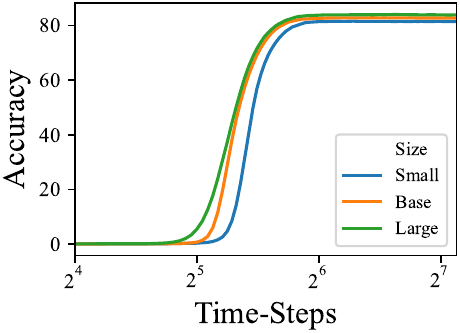}
        \label{fig:subfig3}
    }
    \subfigure[quantization level]{
        \includegraphics[width=0.23\textwidth]{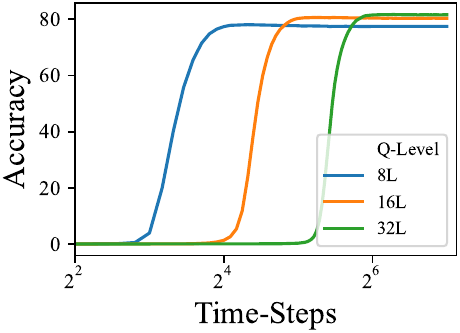}
        \label{fig:subfig4}
    }
    \vskip -0.1in
    \caption{\textbf{Curves of accuracy versus time-step with different settings.} 
    (a) SpikeZIP-TF uses Roberta architecture (QANN is quantized with 32 levels);
    (b) SpikeZIP-TF uses ViT-small (QANN is quantized with 16 levels);
    (c) SpikeZIP-TF use ViT small/base/large as architecture on ImageNet;
    (d) Architecture is ViT-B on ImageNet, where QANNs (\cref{fig:spikezip_arch}) are quantized with different levels.}
    \label{fig:ablation_study}
\end{figure*}

To assess the efficiency of SpikeZIP-TF, we employ the energy model proposed by \citet{cao2015spiking}, which has been utilized in prior works such as \citet{wang2023masked,ding2021optimal}. The model can be expressed as:
\begin{equation}
P = \frac{\textrm{\#total-spikes}}{1 \times 10^{-3}} \times \alpha
\end{equation}
where \#total-spikes is the number of spike activities occurring in SNN during one time-step which takes 1ms in \citet{cao2015spiking} and 1 spike activity consumes $\alpha$ Joules. Unit of $P$ is Watt. 
According to the 45nm hardware \cite{horowitz20141}, we take $\alpha$ as 0.9pJ. 
As summarized in \cref{tab:power}, we use the above power model to compare SpikeZIP-TF with Spikformer \cite{zhou2022spikformer}, Spikformer V2 \cite{zhou2024spikformer}, Spikingformer \cite{zhou2023spikingformer} on ImageNet, as well as SpikeBERT \cite{lv2023spikebert} and SpikeGPT \cite{zhu2023spikegpt} on SST-2. 
During ImageNet inference, although SpikeZIP-TF has higher power consumption than Spikingformer and Spikformer V2 with similar parameters, it achieves lower power with higher accuracy, compared to Spikformer V2. 
This suggests that SpikeZIP-TF can achieve a better power-accuracy trade-off than Spikformer V2.
For SST-2 inference, SpikeZIP-TF exhibits lower power consumption and higher accuracy compared to SpikeBERT, indicating a better power-accuracy trade-off as well.

\subsection{Ablation Study}

\paragraph{Accuracy vs. Time-Steps}
To achieve a better trade-off between accuracy and time-steps in SpikeZIP-TF, we conduct an investigation into various configurations, including the impact of different datasets, model sizes, and quantization levels, with their curves plotted in \cref{fig:ablation_study}. 
Overall, among all the curves in \cref{fig:ablation_study}, there exists a specific time-step called $T_{\rm up}$, beyond which the model's accuracy increases drastically. 
This phenomenon occurs because it requires several time-steps for SpikeZIP-TF to accumulate its output.
\textbf{1) Dataset:} $T_{\rm up}$ increases when 
the dataset becomes harder and more complex. As shown in \cref{fig:ablation_study}(a) and \cref{fig:ablation_study}(b), the $T_{\rm up}$ of SpikeZI-TF with Roberta-B on Chnsenti is much smaller than SST-5 and $T_{\rm up}$ of SpikeZIP-TF with ViT-B on CIFAR10-DVS is smaller than ImageNet. 
\textbf{2) Model Size:} Larger model size leads to better trade-off of accuracy versus time-steps. 
$T_{\rm up}$ of SpikeZIP-TF with ViT on ImageNet decreases when model size becomes large.
The curve of ViT-L in \cref{fig:ablation_study}.(c) is above the curve of ViT-B and ViT-S.
\textbf{3) Quantization Level:} A reduction in quantication level leads to smaller $T_{\rm up}$ values but also results in lower accuracy. 
For SpikeZIP-TF with ViT-B, $T_{\rm up}$ is proportional to quantization level of QViT-B, indicating that the model requires fewer time-steps to complete the inference. 
However, the accuracy of SpikeZIP ViT-S is lower than ViT-L due to the increase of quantization error in QAT. 
Therefore, choosing a suitable quantization level is crucial to strike a balance between accuracy and time-steps in A2S conversion.

\begin{figure}[t]
\centering
\includegraphics[width=\linewidth]{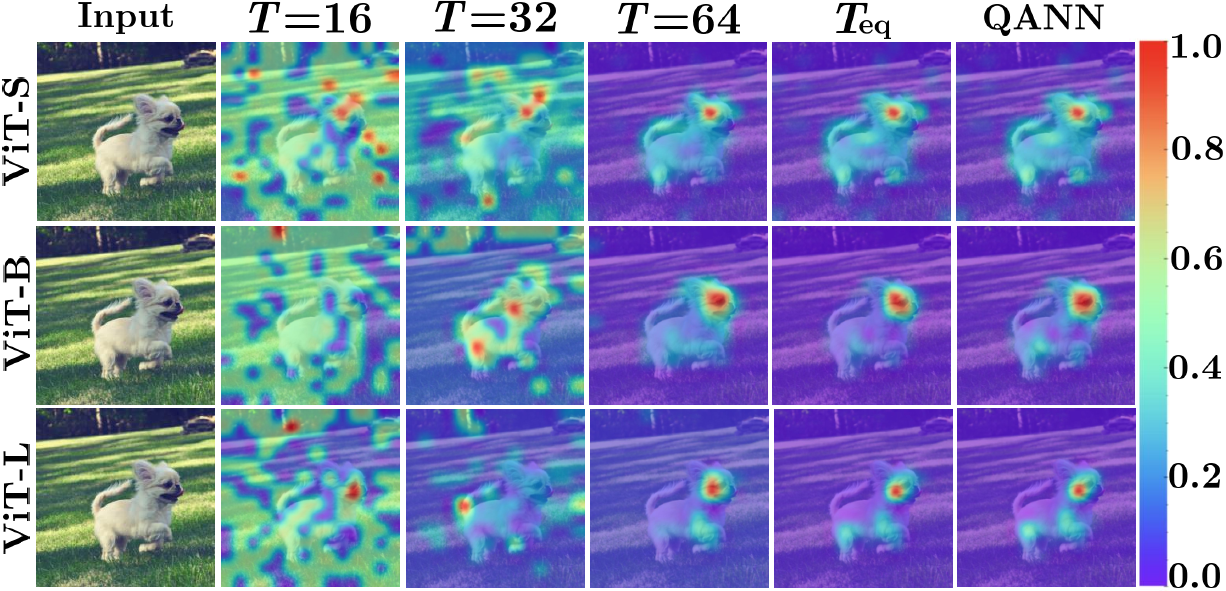}
\caption{
\textbf{Feature map visualization in SpikeZIP-TF (ViT-S/B/L) at different time-steps.} We mask the input with intermediate attention array to visualize the feature map. With the improvement of $T$, feature map in SpikeZIP-TF is more identical to its QANN counterpart.
}
\label{fig:spikezip_visual}
\end{figure}

\subsection{Equivalence Inspection via Experiments}
To further demonstrate that SNN generated by SpikeZIP is functionally equivalent to QANN and takes fewer time-steps for SNN inference, we visualize the evolution of feature maps \textit{w.r.t} different time-steps $T$, as depicted in \cref{fig:spikezip_visual}. The feature maps are obtained by accumulating spiking attention array in SESA and masking the input with the accumulated spiking attention array. We can draw the following observations from \cref{fig:spikezip_visual}: 1) With the spiking accumulation, the feature map of SpikeZIP ViT is gradually close to the corresponding feature map in QANN and final equal to it when $T_{\rm eq}$. 2) When the model size becomes larger, the feature map of SpikeZIP ViT concentrates on the target object at an earlier time-step, which is consistent with the results in \cref{fig:ablation_study}(c).

\section{Conclusion}

SpikeZIP-TF constructs an ANN-to-SNN conversion method that establishes the equivalence between quantized Transformer-based ANN and its SNN counterpart. 
To *make the equivalence framework applicable, we introduce the Spike-Equivalent Self-Attention, Spike-Softmax and Spike-LayerNorm to support the SNN-unfriendly operaters of Transformer-based ANN. Our SpikeZIP-TF leads to state-of-the-art performance on both computer vision, neuromorphic, and natural language understanding tasks. In this work, we mainly focus on ANN-to-SNN conversion method due to its low training cost and nearly loss-less performance between SNN and ANN. We anticipate to extend our SpikeZIP-TF on direct learning methods, which is expected to reduce training cost and achieve promising performance under ultra-low inference time-step.

\section*{Acknowledgments}

This work is partially supported by National Key R\&D Program of China~(2022YFB4500200), National Natural Science Foundation of China~(Nos.62102257, 62306176), Biren Technology–Shanghai Jiao Tong University Joint Laboratory Open Research Fund, Microsoft Research Asia Gift Fund, 
Lingang Laboratory Open Research Fund (No.LG-QS-202202-11), Natural Science Foundation of Shanghai (No. 23ZR1428700), and CCF-Baichuan-Ebtech Foundation Model Fund.

\section*{Impact Statement}
This work is a fundamental research in bridging the state-of-the-art deep neural network and spiking neural network in the community of neuromorphic community. There are many potential societal consequences of our work, none which we feel must be specifically highlighted here.
\bibliographystyle{icml2024}
\bibliography{main}

\begin{thebibliography}{55}
\providecommand{\natexlab}[1]{#1}
\providecommand{\url}[1]{\texttt{#1}}
\expandafter\ifx\csname urlstyle\endcsname\relax
  \providecommand{\doi}[1]{doi: #1}\else
  \providecommand{\doi}{doi: \begingroup \urlstyle{rm}\Url}\fi

\bibitem[Ba et~al.(2016)Ba, Kiros, and Hinton]{ba2016layer}
Ba, J.~L., Kiros, J.~R., and Hinton, G.~E.
\newblock Layer normalization.
\newblock \emph{arXiv preprint arXiv:1607.06450}, 2016.

\bibitem[Bal \& Sengupta(2023)Bal and Sengupta]{bal2023spikingbert}
Bal, M. and Sengupta, A.
\newblock Spikingbert: Distilling bert to train spiking language models using implicit differentiation.
\newblock \emph{arXiv preprint arXiv:2308.10873}, 2023.

\bibitem[Bao et~al.(2022)Bao, Dong, Piao, and Wei]{bao2022beit}
Bao, H., Dong, L., Piao, S., and Wei, F.
\newblock Beit: Bert pre-training of image transformers, 2022.

\bibitem[Bu et~al.(2023)Bu, Fang, Ding, Dai, Yu, and Huang]{bu2023optimal}
Bu, T., Fang, W., Ding, J., Dai, P., Yu, Z., and Huang, T.
\newblock Optimal ann-snn conversion for high-accuracy and ultra-low-latency spiking neural networks.
\newblock \emph{arXiv preprint arXiv:2303.04347}, 2023.

\bibitem[Cao et~al.(2015)Cao, Chen, and Khosla]{cao2015spiking}
Cao, Y., Chen, Y., and Khosla, D.
\newblock Spiking deep convolutional neural networks for energy-efficient object recognition.
\newblock \emph{International Journal of Computer Vision}, 113:\penalty0 54--66, 2015.

\bibitem[Che et~al.(2023)Che, Zhou, Ma, Fang, Chen, Shen, Yuan, and Tian]{che2023auto}
Che, K., Zhou, Z., Ma, Z., Fang, W., Chen, Y., Shen, S., Yuan, L., and Tian, Y.
\newblock Auto-spikformer: Spikformer architecture search.
\newblock \emph{arXiv preprint arXiv:2306.00807}, 2023.

\bibitem[Chen et~al.(2020)Chen, Radford, Child, Wu, Jun, Luan, and Sutskever]{chen2020generative}
Chen, M., Radford, A., Child, R., Wu, J., Jun, H., Luan, D., and Sutskever, I.
\newblock Generative pretraining from pixels.
\newblock In \emph{International conference on machine learning}, pp.\  1691--1703. PMLR, 2020.

\bibitem[Clark et~al.(2020)Clark, Luong, Le, and Manning]{clark2020electra}
Clark, K., Luong, M.-T., Le, Q.~V., and Manning, C.~D.
\newblock Electra: Pre-training text encoders as discriminators rather than generators, 2020.

\bibitem[Cubuk et~al.(2020)Cubuk, Zoph, Shlens, and Le]{cubuk2020randaugment}
Cubuk, E.~D., Zoph, B., Shlens, J., and Le, Q.~V.
\newblock Randaugment: Practical automated data augmentation with a reduced search space.
\newblock In \emph{Proceedings of the IEEE/CVF conference on computer vision and pattern recognition workshops}, pp.\  702--703, 2020.

\bibitem[Davies et~al.(2018)Davies, Srinivasa, Lin, Chinya, Cao, Choday, Dimou, Joshi, Imam, Jain, et~al.]{davies2018loihi}
Davies, M., Srinivasa, N., Lin, T.-H., Chinya, G., Cao, Y., Choday, S.~H., Dimou, G., Joshi, P., Imam, N., Jain, S., et~al.
\newblock Loihi: A neuromorphic manycore processor with on-chip learning.
\newblock \emph{Ieee Micro}, 38\penalty0 (1):\penalty0 82--99, 2018.

\bibitem[Deng et~al.(2009)Deng, Dong, Socher, Li, Li, and Fei-Fei]{deng2009imagenet}
Deng, J., Dong, W., Socher, R., Li, L.-J., Li, K., and Fei-Fei, L.
\newblock Imagenet: A large-scale hierarchical image database.
\newblock In \emph{2009 IEEE conference on computer vision and pattern recognition}, pp.\  248--255. Ieee, 2009.

\bibitem[Devlin et~al.(2018)Devlin, Chang, Lee, and Toutanova]{devlin2018bert}
Devlin, J., Chang, M.-W., Lee, K., and Toutanova, K.
\newblock Bert: Pre-training of deep bidirectional transformers for language understanding.
\newblock \emph{arXiv preprint arXiv:1810.04805}, 2018.

\bibitem[Ding et~al.(2021)Ding, Yu, Tian, and Huang]{ding2021optimal}
Ding, J., Yu, Z., Tian, Y., and Huang, T.
\newblock Optimal ann-snn conversion for fast and accurate inference in deep spiking neural networks.
\newblock \emph{arXiv preprint arXiv:2105.11654}, 2021.

\bibitem[Dosovitskiy et~al.(2020)Dosovitskiy, Beyer, Kolesnikov, Weissenborn, Zhai, Unterthiner, Dehghani, Minderer, Heigold, Gelly, et~al.]{dosovitskiy2020image}
Dosovitskiy, A., Beyer, L., Kolesnikov, A., Weissenborn, D., Zhai, X., Unterthiner, T., Dehghani, M., Minderer, M., Heigold, G., Gelly, S., et~al.
\newblock An image is worth 16x16 words: Transformers for image recognition at scale.
\newblock \emph{arXiv preprint arXiv:2010.11929}, 2020.

\bibitem[Esser et~al.(2019)Esser, McKinstry, Bablani, Appuswamy, and Modha]{DBLP:journals/corr/abs-1902-08153}
Esser, S.~K., McKinstry, J.~L., Bablani, D., Appuswamy, R., and Modha, D.~S.
\newblock Learned step size quantization.
\newblock \emph{CoRR}, abs/1902.08153, 2019.
\newblock URL \url{http://arxiv.org/abs/1902.08153}.

\bibitem[Gholami et~al.(2022)Gholami, Kim, Dong, Yao, Mahoney, and Keutzer]{gholami2022survey}
Gholami, A., Kim, S., Dong, Z., Yao, Z., Mahoney, M.~W., and Keutzer, K.
\newblock A survey of quantization methods for efficient neural network inference.
\newblock In \emph{Low-Power Computer Vision}, pp.\  291--326. Chapman and Hall/CRC, 2022.

\bibitem[Goyal et~al.(2017)Goyal, Doll{\'a}r, Girshick, Noordhuis, Wesolowski, Kyrola, Tulloch, Jia, and He]{goyal2017accurate}
Goyal, P., Doll{\'a}r, P., Girshick, R., Noordhuis, P., Wesolowski, L., Kyrola, A., Tulloch, A., Jia, Y., and He, K.
\newblock Accurate, large minibatch sgd: Training imagenet in 1 hour.
\newblock \emph{arXiv preprint arXiv:1706.02677}, 2017.

\bibitem[Hao et~al.(2023)Hao, Ding, Bu, Huang, and Yu]{hao2023bridging}
Hao, Z., Ding, J., Bu, T., Huang, T., and Yu, Z.
\newblock Bridging the gap between anns and snns by calibrating offset spikes.
\newblock \emph{arXiv preprint arXiv:2302.10685}, 2023.

\bibitem[He et~al.(2022)He, Chen, Xie, Li, Doll\'ar, and Girshick]{He_2022_CVPR}
He, K., Chen, X., Xie, S., Li, Y., Doll\'ar, P., and Girshick, R.
\newblock Masked autoencoders are scalable vision learners.
\newblock In \emph{Proceedings of the IEEE/CVF Conference on Computer Vision and Pattern Recognition (CVPR)}, pp.\  16000--16009, June 2022.

\bibitem[He \& Fan(2019)He and Fan]{he2019simultaneously}
He, Z. and Fan, D.
\newblock Simultaneously optimizing weight and quantizer of ternary neural network using truncated gaussian approximation.
\newblock In \emph{Proceedings of the IEEE/CVF Conference on Computer Vision and Pattern Recognition}, pp.\  11438--11446, 2019.

\bibitem[Hongmin et~al.(2017)Hongmin, Hanchao, Xiangyang, Guoqi, and Luping]{2017CIFAR10}
Hongmin, L., Hanchao, L., Xiangyang, J., Guoqi, L., and Luping, S.
\newblock Cifar10-dvs: An event-stream dataset for object classification.
\newblock \emph{Frontiers in Neuroscience}, 11, 2017.

\bibitem[Horowitz(2014)]{horowitz20141}
Horowitz, M.
\newblock 1.1 computing's energy problem (and what we can do about it).
\newblock In \emph{2014 IEEE international solid-state circuits conference digest of technical papers (ISSCC)}, pp.\  10--14. IEEE, 2014.

\bibitem[Hu et~al.(2023)Hu, Zheng, Jiang, and Pan]{hu2023fast}
Hu, Y., Zheng, Q., Jiang, X., and Pan, G.
\newblock Fast-snn: Fast spiking neural network by converting quantized ann.
\newblock \emph{arXiv preprint arXiv:2305.19868}, 2023.

\bibitem[Huang et~al.(2016)Huang, Sun, Liu, Sedra, and Weinberger]{huang2016deep}
Huang, G., Sun, Y., Liu, Z., Sedra, D., and Weinberger, K.~Q.
\newblock Deep networks with stochastic depth.
\newblock In \emph{Computer Vision--ECCV 2016: 14th European Conference, Amsterdam, The Netherlands, October 11--14, 2016, Proceedings, Part IV 14}, pp.\  646--661. Springer, 2016.

\bibitem[Kim et~al.(2021)Kim, Gholami, Yao, Mahoney, and Keutzer]{kim2021bert}
Kim, S., Gholami, A., Yao, Z., Mahoney, M.~W., and Keutzer, K.
\newblock I-bert: Integer-only bert quantization.
\newblock In \emph{International conference on machine learning}, pp.\  5506--5518. PMLR, 2021.

\bibitem[Krizhevsky et~al.(2009)Krizhevsky, Hinton, et~al.]{krizhevsky2009learning}
Krizhevsky, A., Hinton, G., et~al.
\newblock Learning multiple layers of features from tiny images.
\newblock \emph{technical report}, 2009.

\bibitem[LeCun et~al.(2015)LeCun, Bengio, and Hinton]{lecun2015deep}
LeCun, Y., Bengio, Y., and Hinton, G.
\newblock Deep learning.
\newblock \emph{nature}, 521\penalty0 (7553):\penalty0 436--444, 2015.

\bibitem[Lee et~al.(2016)Lee, Delbruck, and Pfeiffer]{lee2016training}
Lee, J.~H., Delbruck, T., and Pfeiffer, M.
\newblock Training deep spiking neural networks using backpropagation.
\newblock \emph{Frontiers in neuroscience}, 10:\penalty0 508, 2016.

\bibitem[Li et~al.(2022)Li, Ma, and Furber]{li2022quantization}
Li, C., Ma, L., and Furber, S.
\newblock Quantization framework for fast spiking neural networks.
\newblock \emph{Frontiers in Neuroscience}, 16:\penalty0 918793, 2022.

\bibitem[Liu et~al.(2019)Liu, Ott, Goyal, Du, Joshi, Chen, Levy, Lewis, Zettlemoyer, and Stoyanov]{DBLP:journals/corr/abs-1907-11692}
Liu, Y., Ott, M., Goyal, N., Du, J., Joshi, M., Chen, D., Levy, O., Lewis, M., Zettlemoyer, L., and Stoyanov, V.
\newblock Roberta: {A} robustly optimized {BERT} pretraining approach.
\newblock \emph{CoRR}, abs/1907.11692, 2019.
\newblock URL \url{http://arxiv.org/abs/1907.11692}.

\bibitem[Loshchilov \& Hutter(2016)Loshchilov and Hutter]{loshchilov2016sgdr}
Loshchilov, I. and Hutter, F.
\newblock Sgdr: Stochastic gradient descent with warm restarts.
\newblock \emph{arXiv preprint arXiv:1608.03983}, 2016.

\bibitem[Loshchilov \& Hutter(2017)Loshchilov and Hutter]{loshchilov2017decoupled}
Loshchilov, I. and Hutter, F.
\newblock Decoupled weight decay regularization.
\newblock \emph{arXiv preprint arXiv:1711.05101}, 2017.

\bibitem[Lv et~al.(2022)Lv, Xu, and Zheng]{lv2022spiking}
Lv, C., Xu, J., and Zheng, X.
\newblock Spiking convolutional neural networks for text classification.
\newblock In \emph{The Eleventh International Conference on Learning Representations}, 2022.

\bibitem[Lv et~al.(2023)Lv, Li, Xu, Gu, Ling, Zhang, Zheng, and Huang]{lv2023spikebert}
Lv, C., Li, T., Xu, J., Gu, C., Ling, Z., Zhang, C., Zheng, X., and Huang, X.
\newblock Spikebert: A language spikformer trained with two-stage knowledge distillation from bert.
\newblock \emph{arXiv preprint arXiv:2308.15122}, 2023.

\bibitem[Maass(1997)]{maass1997networks}
Maass, W.
\newblock Networks of spiking neurons: the third generation of neural network models.
\newblock \emph{Neural networks}, 10\penalty0 (9):\penalty0 1659--1671, 1997.

\bibitem[Merolla et~al.(2014)Merolla, Arthur, Alvarez-Icaza, Cassidy, Sawada, Akopyan, Jackson, Imam, Guo, Nakamura, et~al.]{merolla2014million}
Merolla, P.~A., Arthur, J.~V., Alvarez-Icaza, R., Cassidy, A.~S., Sawada, J., Akopyan, F., Jackson, B.~L., Imam, N., Guo, C., Nakamura, Y., et~al.
\newblock A million spiking-neuron integrated circuit with a scalable communication network and interface.
\newblock \emph{Science}, 345\penalty0 (6197):\penalty0 668--673, 2014.

\bibitem[Neftci et~al.(2019)Neftci, Mostafa, and Zenke]{neftci2019surrogate}
Neftci, E.~O., Mostafa, H., and Zenke, F.
\newblock Surrogate gradient learning in spiking neural networks: Bringing the power of gradient-based optimization to spiking neural networks.
\newblock \emph{IEEE Signal Processing Magazine}, 36\penalty0 (6):\penalty0 51--63, 2019.

\bibitem[Pang \& Lee(2004)Pang and Lee]{pang2004sentimental}
Pang, B. and Lee, L.
\newblock A sentimental education: Sentiment analysis using subjectivity summarization based on minimum cuts.
\newblock \emph{arXiv preprint cs/0409058}, 2004.

\bibitem[Roy et~al.(2019)Roy, Jaiswal, and Panda]{roy2019towards}
Roy, K., Jaiswal, A., and Panda, P.
\newblock Towards spike-based machine intelligence with neuromorphic computing.
\newblock \emph{Nature}, 575\penalty0 (7784):\penalty0 607--617, 2019.

\bibitem[Rueckauer et~al.(2017)Rueckauer, Lungu, Hu, Pfeiffer, and Liu]{rueckauer2017conversion}
Rueckauer, B., Lungu, I.-A., Hu, Y., Pfeiffer, M., and Liu, S.-C.
\newblock Conversion of continuous-valued deep networks to efficient event-driven networks for image classification.
\newblock \emph{Frontiers in neuroscience}, 11:\penalty0 682, 2017.

\bibitem[Shrestha \& Orchard(2018)Shrestha and Orchard]{shrestha2018slayer}
Shrestha, S.~B. and Orchard, G.
\newblock Slayer: Spike layer error reassignment in time.
\newblock \emph{Advances in neural information processing systems}, 31, 2018.

\bibitem[Socher et~al.(2013)Socher, Perelygin, Wu, Chuang, Manning, Ng, and Potts]{socher2013recursive}
Socher, R., Perelygin, A., Wu, J., Chuang, J., Manning, C.~D., Ng, A.~Y., and Potts, C.
\newblock Recursive deep models for semantic compositionality over a sentiment treebank.
\newblock In \emph{Proceedings of the 2013 conference on empirical methods in natural language processing}, pp.\  1631--1642, 2013.

\bibitem[Steiner et~al.(2021)Steiner, Kolesnikov, Zhai, Wightman, Uszkoreit, and Beyer]{DBLP:journals/corr/abs-2106-10270}
Steiner, A., Kolesnikov, A., Zhai, X., Wightman, R., Uszkoreit, J., and Beyer, L.
\newblock How to train your vit? data, augmentation, and regularization in vision transformers.
\newblock \emph{CoRR}, abs/2106.10270, 2021.
\newblock URL \url{https://arxiv.org/abs/2106.10270}.

\bibitem[Szegedy et~al.(2016)Szegedy, Vanhoucke, Ioffe, Shlens, and Wojna]{szegedy2016rethinking}
Szegedy, C., Vanhoucke, V., Ioffe, S., Shlens, J., and Wojna, Z.
\newblock Rethinking the inception architecture for computer vision.
\newblock In \emph{Proceedings of the IEEE conference on computer vision and pattern recognition}, pp.\  2818--2826, 2016.

\bibitem[Vaswani et~al.(2017)Vaswani, Shazeer, Parmar, Uszkoreit, Jones, Gomez, Kaiser, and Polosukhin]{vaswani2017attention}
Vaswani, A., Shazeer, N., Parmar, N., Uszkoreit, J., Jones, L., Gomez, A.~N., Kaiser, {\L}., and Polosukhin, I.
\newblock Attention is all you need.
\newblock \emph{Advances in neural information processing systems}, 30, 2017.

\bibitem[Wang et~al.(2023)Wang, Fang, Cao, Zhang, Wang, and Xu]{wang2023masked}
Wang, Z., Fang, Y., Cao, J., Zhang, Q., Wang, Z., and Xu, R.
\newblock Masked spiking transformer.
\newblock In \emph{Proceedings of the IEEE/CVF International Conference on Computer Vision}, pp.\  1761--1771, 2023.

\bibitem[Xu et~al.(2022)Xu, Zhang, Hou, Gong, Wen, Wang, and Zhang]{xu2022delving}
Xu, Z., Zhang, M., Hou, J., Gong, X., Wen, C., Wang, C., and Zhang, J.
\newblock Delving into transformer for incremental semantic segmentation, 2022.

\bibitem[Yao et~al.(2023)Yao, Hu, Zhou, Yuan, Tian, Xu, and Li]{yao2023spike}
Yao, M., Hu, J., Zhou, Z., Yuan, L., Tian, Y., Xu, B., and Li, G.
\newblock Spike-driven transformer.
\newblock \emph{arXiv preprint arXiv:2307.01694}, 2023.

\bibitem[Yun et~al.(2019)Yun, Han, Oh, Chun, Choe, and Yoo]{yun2019cutmix}
Yun, S., Han, D., Oh, S.~J., Chun, S., Choe, J., and Yoo, Y.
\newblock Cutmix: Regularization strategy to train strong classifiers with localizable features.
\newblock In \emph{Proceedings of the IEEE/CVF international conference on computer vision}, pp.\  6023--6032, 2019.

\bibitem[Zhang et~al.(2017)Zhang, Cisse, Dauphin, and Lopez-Paz]{zhang2017mixup}
Zhang, H., Cisse, M., Dauphin, Y.~N., and Lopez-Paz, D.
\newblock mixup: Beyond empirical risk minimization.
\newblock \emph{arXiv preprint arXiv:1710.09412}, 2017.

\bibitem[Zhou et~al.(2023)Zhou, Yu, Zhou, Zhang, Ma, Zhou, and Tian]{zhou2023spikingformer}
Zhou, C., Yu, L., Zhou, Z., Zhang, H., Ma, Z., Zhou, H., and Tian, Y.
\newblock Spikingformer: Spike-driven residual learning for transformer-based spiking neural network.
\newblock \emph{arXiv preprint arXiv:2304.11954}, 2023.

\bibitem[Zhou et~al.(2021)Zhou, Song, Chen, Zhou, Wang, Yuan, and Zhang]{zhou2021rethinking}
Zhou, H., Song, L., Chen, J., Zhou, Y., Wang, G., Yuan, J., and Zhang, Q.
\newblock Rethinking soft labels for knowledge distillation: A bias-variance tradeoff perspective, 2021.

\bibitem[Zhou et~al.(2022)Zhou, Zhu, He, Wang, Yan, Tian, and Yuan]{zhou2022spikformer}
Zhou, Z., Zhu, Y., He, C., Wang, Y., Yan, S., Tian, Y., and Yuan, L.
\newblock Spikformer: When spiking neural network meets transformer.
\newblock \emph{arXiv preprint arXiv:2209.15425}, 2022.

\bibitem[Zhou et~al.(2024)Zhou, Che, Fang, Tian, Zhu, Yan, Tian, and Yuan]{zhou2024spikformer}
Zhou, Z., Che, K., Fang, W., Tian, K., Zhu, Y., Yan, S., Tian, Y., and Yuan, L.
\newblock Spikformer v2: Join the high accuracy club on imagenet with an snn ticket.
\newblock \emph{arXiv preprint arXiv:2401.02020}, 2024.

\bibitem[Zhu et~al.(2023)Zhu, Zhao, and Eshraghian]{zhu2023spikegpt}
Zhu, R.-J., Zhao, Q., and Eshraghian, J.~K.
\newblock Spikegpt: Generative pre-trained language model with spiking neural networks.
\newblock \emph{arXiv preprint arXiv:2302.13939}, 2023.

\end{thebibliography}

\clearpage
\newpage
\appendix



\title{Appendix}
\setcounter{page}{1}
\setcounter{section}{0}
\setcounter{table}{0}
\setcounter{figure}{0}
\setcounter{equation}{0}
\renewcommand{\thesection}{A\arabic{section}}
\renewcommand{\thetable}{A\arabic{table}}
\renewcommand{\thefigure}{A\arabic{figure}}
\renewcommand{\theequation}{A\arabic{equation}}

\begin{table*}[h]
\caption{Experiments with Swin-Transformer Tiny on A2S methods and corresponding digital transformers.}
\label{tab:swin}
\resizebox{\linewidth}{!}{
\begin{tabular}{llcccccccc}
\hline
&&\multicolumn{2}{c}{ImageNet}&\multicolumn{2}{c}{Cifar-100}&\multicolumn{2}{c}{Cifar-10}&\multicolumn{2}{c}{Cifar-10-DVS}\\
\multirow{-2}{*}{Method}&\multirow{-2}{*}{Category}&Acc/\#T&Power(W)&Acc/\#T&Power(W)&Acc/\#T&Power(W)&Acc/\#T&Power(W)\\\hline
MST&QANN&{80.51/1}&{20.810}&{88.72/1}&{20.730}&{98.14/1}&{20.730}&{88.98/1}&{20.940}\\
MST&SNN&{78.51/512}&{8.528}&{86.91/256}&{8.286}&{97.27/256}&{8.304}&{88.12/512}&{8.188}\\\hline
SpikeZIP-TF(ours)&QANN&{80.70/1}&{20.170}&{87.94/1}&{20.170}&{98.38/1}&{20.170}&{90.50/1}&{20.280}\\
SpikeZIP-TF(ours)&SNN&{\textbf{80.74/64}}&{\textbf{1.363}}&{\textbf{87.91/32}}&{\textbf{1.428}}&{\textbf{98.45/32}}&{\textbf{1.422}}&{\textbf{90.40/64}}&{\textbf{1.317}}\\\hline
\end{tabular}
}
\end{table*}

\section{Experimental Results on Swin-Transformer}
In \cref{tab:swin}, we conduct comprehensive experiments with Swin-Transformer Tiny network, on ImageNet, CIFAR-100, CIFAR-10 and CIFAR-10-DVS. The power consumption is calculated by Equation (6) in the manuscript, which is adopted in MST as well. Compared with the digital Transformers (QANN) counterpart, our SpikeZIP-TF achieves on-par accuracy with lower power consumption. Note that, the minor accuracy difference between QANN and SNN in SpikeZIP-TF are resulted from the GPU numeric error. For MST (the closest peer of SpikeZIP-TF), SNN suffers not only the distinguishable accuracy degradation from its QANN, but also the power reduction is lower than that of SpikeZIP-TF.

\section{Calculation of Total Spikes}
The $total~spikes$ means the total spike activity during one time-step, whose type includes the pre-synaptic (\textit{i.e.}, delivered by synapases) and post-synaptic (\textit{i.e.}, generated by neurons). The number of post-synaptic spikes $N_{\rm post}$ is calculated by \cref{eqt:post}.
\begin{equation}
    N_{\rm post} = \sum_{i=0}^l R^i \times N_{\rm neu}^i
    \label{eqt:post}
\end{equation}
where $R$ is the firing rate of the $i$-th layer and $N_{\rm neu}^i$ is the number of neurons of the $i$-th layer. $N_{\rm pre}$ reflects the update operation in neurons. Then for the number of pre-synaptic spikes, the calculation depends on the synaptic connection structure, which is modeled as:
\cref{eqt:pre}.
\begin{equation}
    N_{\rm pre} = \sum_{i=0}^l R^{i-1} \times f_{\rm in}^i \times N_{\rm neu}^i
    \label{eqt:pre}
\end{equation}
where $f_{\rm in}^i$ is the number of fan-in operation of the $i$-t layer. $f_{\rm in}^i$ depends on the connection struture of between neuron layers. For the convolution layer, $f_{{\rm in},{\rm conv}}^i$ is calculated by:
\begin{equation}
\begin{aligned}
    f_{{\rm in},{\rm conv}} = C_{\rm in} \times K_{\rm H} \times K_{\rm W}
\end{aligned}
\label{eqt:fin_conv}
\end{equation}
where $C_{\rm in}$ is the input channel; $K_{\rm H}$ and $K_{\rm W}$ are the height and width of kernel; Therefore, for convolution layer, the pre-synaptic spikes $N_{\rm pre}$ is:
\begin{equation}
\begin{aligned}
    N_{{\rm pre},{\rm conv}} = C_{\rm in} \times K_{\rm H} \times K_{\rm W} \times C_{\rm out} \times O_{\rm H} \times O_{\rm W}
\end{aligned}
\label{eqt:pre_conv}
\end{equation}
where $C_{\rm out}$ is the output channel; $O_{\rm H}$ and $O_{\rm W}$ are the height and width of output feature. In \cref{eqt:pre_conv}, $C_{\rm in} \times K_{\rm H} \times K_{\rm W}$ is the fan-in of one neuron and $C_{\rm out} \times O_{\rm H} \times O_{\rm W}$ is the number of neurons in one neuron layer. It is worth noted the relationship between $f_{{\rm in},{\rm conv}}$ and the FLOPS of convolution is $2 \times f_{{\rm in},{\rm conv}} = {\rm FLOPS}_{\rm conv}$. For the linear layer, $f_{{\rm in},{\rm fc}}$ is the number of input neurons of $N_{\rm in}$. The pre-synaptic spikes of the linear layer $N_{{\rm pre},{\rm fc}}$ is:
\begin{equation}
\begin{aligned}
    N_{{\rm pre},{\rm fc}} = N_{\rm in} \times N_{\rm out}
\end{aligned}
\label{eqt:pre_linear}
\end{equation}
where $N_{\rm in}$ and $N_{\rm out}$ are the numbers of input and output features of the linear layer. After calculating the pre-synaptic and post-synaptic spikes, the total spikes is equal to the sum of pre-synaptic and post-synaptic spikes:
\begin{equation}
    \textrm{\#total-spikes} = N_{\rm pre} + N_{\rm post}
\label{eqt:total_spikes}
\end{equation}

\section{Proof of Equivalence:}
\subsection{Preliminaries}
\paragraph{Notations definition.}
For reader-friendly, we provide the notations used in the following proof in \cref{tab:notation_table1}.

\begin{table}[t]
\centering
\resizebox{\linewidth}{!}{
\renewcommand{\arraystretch}{1.0}
\begin{tabular}{cl}
\toprule
Notation & \multicolumn{1}{c}{Description} \\ 
\midrule
$V_t$ & \begin{tabular}[c]{@{}l@{}}potential of neuron membrane at time-step $t$\end{tabular} \\
$V_{\textrm{thr}}$ & \begin{tabular}[c]{@{}l@{}}threshold voltage for neuron to fire a spike\end{tabular} \\
$V^{\textrm{in}},V^{\textrm{out}}$ & \begin{tabular}[c]{@{}l@{}}input or output voltage of neuron \end{tabular} \\
$T_{\textrm{eq}}$ & \begin{tabular}[c]{@{}l@{}}time-step that neuron enters equilibrium state\end{tabular} \\
$T_{\textrm{off}}$ & \begin{tabular}[c]{@{}l@{}}time-steps when input and bias are turned off\end{tabular} \\
$S_{t}$ & \begin{tabular}[c]{@{}l@{}}spike tracer at time-step $t$\end{tabular} \\
$S_{\textrm{max}} / S_{\textrm{min}}$ & \begin{tabular}[c]{@{}l@{}}maximum/minimum value in spike tracer \end{tabular} \\
$\textrm{clip}(x,\alpha_\textrm{min},\alpha_\textrm{max})$ & \begin{tabular}[c]{@{}l@{}}clip function that limits $x$ between  $\alpha_\textrm{min}$ and $\alpha_\textrm{max}$\end{tabular} \\
$\textrm{floor}(x)$ & floor function that round down $x$ \\
$\Theta(V,V_{\textrm{thr}},S)$ & output spike decision function of ST-BIF neuron 
\\
$\mQ_t, \mK_t$ & spiking Query and spiking Key matrix in SESA. \\
$\mS_{Q,t}, \mS_{K,t}$ & accumulated spike trains of query and key. \\
\bottomrule
\end{tabular}
}
\caption{Summary of mathematical notations used in the proof.}
\label{tab:notation_table1}
\end{table}

\paragraph{ST-BIF Neuron Model}
\par The definition of the ST-BIF neuron model is:
\begin{equation}
\label{eqt:ST-BIF1}
\begin{gathered}
V_t = 
V_{t-1} + V^\textrm{in}_t - V_\textrm{thr} \cdot \Theta(V_{t-1} + V^\textrm{in}_t, V_\textrm{thr}, S_{t-1}) \\
S_t = S_{t-1} + \Theta(V_{t-1} + V^\textrm{in}_t, V_\textrm{thr}, S_{t-1}) \\
\Theta(V, V_\textrm{thr}, S) = 
\begin{cases}
1 ;& V \geq V_\textrm{thr} ~ \& ~ S < S_{\textrm{max}} \\
0 ;&  \textrm{other} \\
-1 ;&  V < 0 ~ \& ~ S > S_{\rm min}
\end{cases}
\end{gathered}
\end{equation}
where the first equation of \cref{eqt:ST-BIF1} depicts the membrane potential updating in ST-BIF neuron. The membrane potential at time-step $t$ equals to the membrane potential at the prior time-step $t-1$ adding the potential $V^\textrm{in}_t$ caused by the input charge at $t$ time-step, then subtract the potential of the fired spike. The fired spike is recorded by the spike tracer defined in the second equation of \cref{eqt:ST-BIF1}. The firing behavior of ST-BIF neuron depends on the spike decision function $\Theta$ in the third equation.

\paragraph{Equilibrium State}
Assume the external stimulate (\textit{e.g.}, input and bias) are applied to SNN from $T=0$ to $T_{\rm off}$, we define the \textbf{equilibrium state} of SNN as the status where neurons of entire SNN are static (\textit{e.g.}, no further activities of neuron firing and membrane update). 
The time-step that SNN enters the equilibrium state is noted as $T_{\rm eq}$.

\subsection{The Equivalence between Quantized Function and ST-BIF$^{+}$ Neuron}

\begin{lemma}
\label{sec:lemma1}
After entering the equilibrium state at $T_{\rm eq}$, the accumulated output spikes of one ST-BIF neuron can be derived as a closed-form equation of quantization function:
\begin{equation}
\label{eqt:IO-projection-of-BIF-neuron2}
    V^{\rm out} = V_{\rm thr} \cdot {\rm clip}( {\rm floor}( \dfrac{V^{\rm in} + V_{t=0}}{V_{\rm thr}}) , S_{{\rm min}} , S_{{\rm max}}) 
\end{equation}
where $V^\textrm{in} = \sum_{t=0}^{T_{\rm eq}} V^\textrm{in}_t$ is the accumulated input until $T_\textrm{eq}$, and $V_{t=0}$ denotes the initial membrane potential. 
\end{lemma}

\begin{proof}

\par Starting from the first equation in \cref{eqt:ST-BIF1}, the membrane potential can be calculated without using the recursive form by summing over simulated time $T$:
\begin{equation}
\label{eqt:V_sum}
    V_T - V_0 = \sum^T_{t=1}V^{\textrm{in}}_t - V_{\textrm{thr}}\cdot\sum^T_{t=1}\Theta(V_{t-1}+V_t^{\textrm{in}},V_{\textrm{thr}},S_{t-1}) 
\end{equation}
We sum the spike tracer $S_t$ in~\cref{eqt:ST-BIF1} over the inference time-steps $T$, which is described as:
\begin{equation}
\label{eqt:S_sum}
    S_T - S_0 =  \sum^T_{t=1}\Theta(V_{t-1}+V_t^{\textrm{in}},V_{\textrm{thr}},S_{t-1})
\end{equation}
where $S_0 = 0$ is the default setting. By substituting \cref{eqt:S_sum} into \cref{eqt:V_sum}, \cref{eqt:V_sum} is simplified as:
\begin{equation}
\label{eqt:SV_combine}
    V_T = (\sum^T_{t=1}V^{\textrm{in}}_t + V_0) - V_{\textrm{thr}}\cdot S_T
\end{equation}
Then, we divide both sides of \cref{eqt:SV_combine} by $V_{\textrm{thr}}$. With additional simple transformation, we get:
\begin{equation}
\label{eqt:S_calculate}
    S_T = \frac{(\sum^T_{t=1}V^{\textrm{in}}_t + V_0 - V_T)}{V_{\textrm{thr}}}
\end{equation}
Hereby, we discuss three cases about $S_T$ in \cref{eqt:S_calculate} as follows.

\paragraph{Case 1.} $ S_\textrm{min} \leq (\sum^T_{t=1}V^{\textrm{in}}_t + V_0)/V_{\rm thr} \leq S_\textrm{max}$: When $T \geq T_{\textrm{eq}}$, according to the definition of equilibrium state, the membrane potential of the ST-BIF neuron is insufficient to fire a spike, which means $V_T < V_{\textrm{thr}}$. Since $S_T$ is an integer in \cref{eqt:S_calculate}, based on the definition of round down function~(e.g., floor), $S_T$ can be rewritten as:
\begin{equation}
\label{eqt:S_case1}
\begin{gathered}
    S_T = \textrm{floor}(\frac{\sum^T_{t=1}V^{\textrm{in}}_t + V_0}{V_{\textrm{thr}}})
\end{gathered}
\end{equation}
where the error caused by the rounding (down) of the membrane potential in \cref{eqt:S_case1} is equal to $V_T$. \cref{eqt:S_case1} represents the discretization part in the quantized-ReLU (Q-ReLU) function.

\paragraph{Case 2.} $(\sum^T_{t=1}V^{\textrm{in}}_t + V_0)/V_{\rm thr} > S_\textrm{max}$: According to the firing decision function $\Theta$, the ST-BIF neuron fires positive spikes until spike tracer $S_T = S_{\textrm{max}}$, then the $S_{T}$ becomes static:
\begin{equation}
\label{eqt:S_case2}
\begin{gathered}
    S_T = S_{\textrm{max}}
\end{gathered}
\end{equation}
In virtue of setting the upper bound of $S_T$, we successfully limit the accumulated output in the ST-BIF neuron to $S_\textrm{max}$, which corresponds to the clipping upper bound in Q-ReLU.

\paragraph{Case 3.} $(\sum^T_{t=1}V^{\textrm{in}}_t + V_0)/V_{\rm thr} < S_\textrm{min}$: Similar to Case 2, before the $T_{\rm eq}$, the ST-BIF neuron fires negative spikes until spike tracer $S_T = S_{\rm min}$, then $S_T$ is fixed.
\begin{equation}
\label{eqt:S_case3}
\begin{gathered}
    S_T = S_{\textrm{min}}
\end{gathered}
\end{equation}
Then, we leverage the clip function to combine \cref{eqt:S_case1}, \cref{eqt:S_case2} and \cref{eqt:S_case3}, then we can derive:
\begin{equation}
\label{eqt:S_combine}
    S_T = \textrm{clip}(\textrm{floor}(\frac{\sum_{t=1}^T V_t^{\textrm{in}} + V_0}{V_{\textrm{thr}}}),S_{\textrm{min}},S_{\textrm{max}})
\end{equation}
The total output of an ST-BIF neuron can be defined as:
\begin{equation}
\label{eqt:V_out}
    V^{\textrm{out}} = \sum^{T}_{t=1} V^{\textrm{out}}_t = V_{\textrm{thr}} \cdot \sum^{T}_{t=1}\Theta(V_{t-1}+V^{\textrm{in}}_t,V_{\textrm{thr}},S_{t-1})
\end{equation}
where $\sum^{T}_{t=1}\Theta(V_{t-1}+V^{\textrm{in}}_t, V_\textrm{thr}, S_{t-1})$ denotes the number of total output spikes of an ST-BIF neuron. \cref{eqt:V_out} shows the accumulated output of the ST-BIF neuron is equal to the number of total output spikes scaled by the firing threshold $V_{\textrm{thr}}$. We substitute \cref{eqt:S_sum} in \cref{eqt:V_out} and get:
\begin{equation}
\label{eqt:Q_out}
    V^{\textrm{out}} = V_{\textrm{thr}} \cdot S_T    
\end{equation}
Then, we further substitute \cref{eqt:S_combine} in \cref{eqt:Q_out}:
\begin{equation}
\label{eqt:final}
    V_{\textrm{out}} = V_{\textrm{thr}} \cdot \textrm{clip}(\textrm{floor}(\frac{\sum_{t=1}^T V_t^{\textrm{in}} + V_0}{V_{\textrm{thr}}}),0,S_{\textrm{max}})
\end{equation}
Proof complete.

\end{proof}

\subsection{The Equivalence of Spike-Equivalent Self-Attention (SESA)}

\paragraph{Dynamic Model of SESA.}
Before we prove the equivalence between the SESA and quantized self-attention shown in \cref{fig:spikezip_arch}, we introduce the dynamic model of SESA during the SNN inference:
\begin{equation}
    \begin{gathered}        
    \mS_{{\rm Q},t} = \mS_{{\rm Q},t-1} + \mQ_t, ~~
    \mS_{{\rm K},t} = \mS_{{\rm K},t-1} + \mK_t \\
    \mO_t = \mS_{{\rm Q},t} \cdot \mK_t^T + \mQ_{t} \cdot \mS_{{\rm K}, t}^T - \mQ_{t} \cdot \mK_t^T
    \end{gathered}
\end{equation}
where $O_t$ is the output of SESA at $t$ time-step, other notations are summarized in \cref{tab:notation_table1}. As shown in \cref{fig:matrix_precess}, we calculate the output of SESA at $t$ time-step (the green dotted frame) by doing three matrix multiplication. 

\begin{lemma}
\label{sec:lemma2}
After entering the equilibrium state $T \geq T_{\rm eq}$, the accumulated output of Spiking-Equivalent Self-Attention (SESA) equals the output of Quantized Self-Attention (QSA) with same input:
\begin{equation}
    \sum_{t=0}^{T_{\rm eq}} \mO_t = \mO_q
    \label{eqt:Lemma2}
\end{equation}
Where $\mO_t$ and $\mO_q$ are the accumulated output of SESA and the output of QSA.
\end{lemma}

\begin{proof}
\label{sec:lemma2_proof}
Firstly, according to the formula of QSA n \cref{eqt:multi2}, the output of QSA $\mO_q$ can be written:
\begin{equation}
    \LHS = \mO_q = \mQ_{\textrm{q}} \cdot \mK_{\textrm{q}}
    \label{eqt:qsa_def}
\end{equation}
where $\mQ_{\textrm{q}}$, $\mK_{\textrm{q}}$ are the query and key matrices in QSA, which are also the outputs of the quantized functions. In SESA, these quantized functions are replaced by the ST-BIF$^+$ neurons. By leveraging the \cref{sec:lemma1}, we build the relation between $\mQ_{\textrm{q}}$, $\mK_{\textrm{q}}$ and the spiking query $\mQ_{\textrm{t}}$ and spiking key $\mK_{\textrm{t}}$:
\begin{equation}
    \LHS = \mQ_{\textrm{q}} \cdot \mK_{\textrm{q}} = \sum_{t=0}^{T_{\textrm{eq}}} \mQ_{t} \cdot \sum_{t=0}^{T_{\textrm{eq}}} \mK_{t}
    \label{eqt:qsa_def}
\end{equation}
\begin{figure}[t]
\centering
\includegraphics[width=\linewidth]{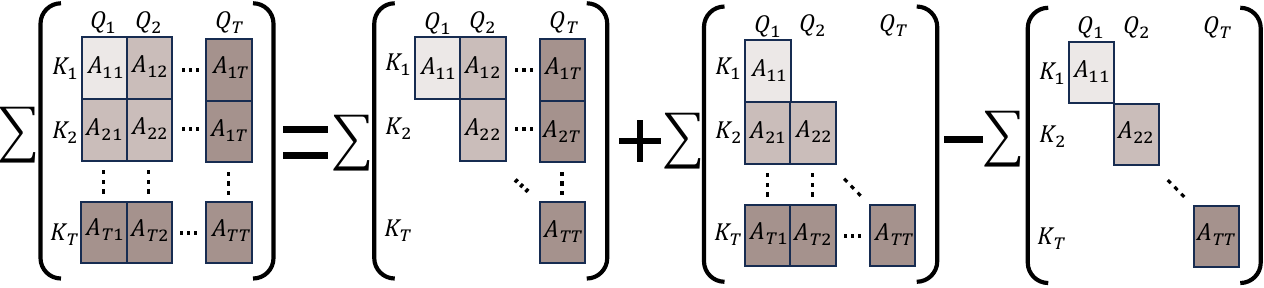}
\caption{The Decomposition of Summing Up.} 
\label{fig:sum_up}
\end{figure}
where $\sum_{t=0}^{T_{\textrm{eq}}} \mQ_{t} \cdot \sum_{t=0}^{T_{\textrm{eq}}} \mK_{t}$ can be considered summing every element of a 2-D matrix up, whose element $A_{i,j} = \mQ_{i} \cdot \mK_{j}$. The procedure is also shown in the left brown matrix in \cref{fig:sum_up}. Then, as illustrated in \cref{fig:sum_up}, we decompose the sum operation into three parts: upper triangular, lower triangular, and diagonal. Therefore, we can rewrite the  $\sum_{t=0}^{T_{\textrm{eq}}} \mQ_{t} \cdot \sum_{t=0}^{T_{\textrm{eq}}} \mK_{t}$ as:
\begin{equation}
\begin{aligned}
    &\LHS = \sum_{t=0}^{T_{\textrm{eq}}} \mQ_{t} \cdot \sum_{t=0}^{T_{\textrm{eq}}} \mK_{t} = \\ &\sum_{t=0}^{T_{\textrm{eq}}} \mQ_{t} \sum_{t_1=0}^{t} \mK_{t_1}^T + \sum_{t=0}^{T_{\textrm{eq}}} \sum_{t_1=0}^{t} \mQ_{t_1} \mK_{t}^T - \sum_{t=0}^{T_{\textrm{eq}}} \mQ_{t} \mK_{t}^T = \\
    &\sum_{t=0}^{T_{\textrm{eq}}}(\mQ_{t} \sum_{t_1=0}^{t} \mK_{t_1}^T +\sum_{t_1=0}^{t} \mQ_{t_1} \mK_{t}^T - \mQ_{t} \mK_{t}^T)
    \label{eqt:SESA_equal}
\end{aligned}
\end{equation}
where the $\sum_{t_1=0}^{t} \mQ_{t_1}$ and $\sum_{t_1=0}^{t} \mK_{t_1}^T$ are the accumulated spike trains of query and key at $t$ time-step, which equal the $\mS_{{\rm Q},t}$ and $\mS_{{\rm K},t}$:
\begin{equation}
\begin{aligned}
    &\LHS = \sum_{t=0}^{T_{\textrm{eq}}}(\mQ_{t} 
\mS_{{\rm K},t}^T + \mS_{{\rm Q},t} \mK_{t}^T - \mQ_{t} \mK_{t}^T) = \\
    & \sum_{t=0}^{T_{\textrm{eq}}} \mO_{t} = \RHS
    \label{eqt:SESA_equal1}
\end{aligned}
\end{equation}
\cref{eqt:SESA_equal1} prove the output of QSA ($\LHS$) equals the accumulated output of SESA ($\RHS$). Proof complete.
\end{proof}
\subsection{The Equivalence of Spike-Softmax and Spike-LayerNorm.} 
\paragraph{Dynamic Model.} The dynamic model of Spike-Softmax and Spike-LayerNorm at each time-step are inspired by the differential algorithm, which can be written as:
\begin{equation}
\begin{gathered}    
\quad X_{t} = X_{t-1} + x_{t}; \quad
O_{t} = {\rm \sigma}(X_{t}) \\
o_{t} = O_{t} - O_{t-1} 
\label{eqt:SS_SL_model}
\end{gathered}
\end{equation}
where $X_t$ is the accmulated input during $t$ time-step, $x_t$ is the input at $t$ time-step, $\sigma$ is the Softmax/LayerNorm function, $o_t$ is the output of Spike-Softmax and Spike-LayerNorm at $t$ time-step. \cref{eqt:SS_SL_model} decompose the activation in Softmax and LayerNorm into multiple time-steps without changing the summation.
\begin{lemma}
    If the accumulated input in SNN equals to the input in QANN, the accumulated output of Spike-Softmax/Spike-LayerNorm equals the output of Softmax/LayerNorm after entering equilibruim state:
    \begin{equation}
        \sum_{t=0}^{T_{\rm eq}} o_t = o_q; \quad {\rm s.t.} \sum_{t=0}^{T_{\rm eq}} x_t = x_q
    \end{equation}
    \label{Lemma3}
    where $x_q$ and $o_q$ the input and output of Softmax (LayerNorm). 
\end{lemma}

\begin{proof}
In \cref{eqt:SS_SL_model}, summing up the $X_t$ through time, we have $X_{T_{\rm eq}} = \sum_{t=0}^{T_{\rm eq}} x_t = x_q$. Similarly, summing up the $o_t$ over time, we also have $\sum_{t=0}^{T_{\rm eq}} o_t = O_{T_{\rm eq}}$. Combine the two condition, we have:
\begin{equation}
\begin{aligned}
    &\LHS = o_q = \sigma(x_q) = \sigma(\sum\nolimits_{t=0}^{T_{\rm eq}} x_t) = \\ &\sigma(X_{T_{\rm eq}}) = O_{T_{\rm eq}} = \sum\nolimits_{t=0}^{T_{\rm eq}} o_t = \RHS
\end{aligned}
\end{equation}
Proof complete.
\end{proof}


\section{Implementation Details}

\begin{table}[t]
\caption{The hyperparameter of end-to-end finetuning with ViT-S ReLU on ImageNet. $\ast$:~\cite{loshchilov2017decoupled}. $\star$:~\cite{chen2020generative}. $\dagger$:~\cite{clark2020electra, bao2022beit}. $\ddag$:~\cite{goyal2017accurate}. $\triangledown$ :~\cite{loshchilov2016sgdr}. $\vartriangle$:~\cite{cubuk2020randaugment}. $\blacktriangle$:~\cite{szegedy2016rethinking}. $\blacktriangledown$:~\cite{zhang2017mixup}. $\lozenge$:~\cite{yun2019cutmix}. $\blacklozenge$:~\cite{huang2016deep}.}
\label{tab:vit-s_ft_imagenet} 
\centering
\begin{tabular}{l|l}
config & value~(ViT-S ReLU)) \\ \hline
optimizer & AdamW~$\ast$ \\
optimizer momentum & $\beta_1,\beta_2$ = 0.9, 0.99~$\star$ \\
base learning rate & 1e-4 \\
weight decay & 0.05 \\
layer-wise lr decay~$\dagger$ & 0.65 \\
GPUs & 4 \\
batch\_size & 64/256 \\
warmup epochs~$\ddag$ & 5 \\
training epochs & 100 \\
learning rate schedule & cosine decay~$\triangledown$ \\
distillation weight & 1.0 \\
distillation temp. & 2.0 \\
augmentation & RandomAug~(9, 0.5)~$\vartriangle$ \\
label smoothing~$\blacktriangle$ & 0.1 \\
mixup~$\blacktriangledown$ & 0.8 \\
cutmix~$\lozenge$ & 1.0 \\
drop path~$\blacklozenge$ & 0.1 \\
\end{tabular}
\end{table}

\begin{table}[h]
\caption{The hyperparameter of end-to-end finetuning with ViT-B ReLU on ImageNet.}
\label{tab:vit-b_ft_imagenet} 
\centering
\begin{tabular}{l|l}
config & value~(ViT-B ReLU)) \\ \hline
optimizer & AdamW \\
optimizer momentum & $\beta_1,\beta_2$ = 0.9, 0.99 \\
base learning rate & 1.66e-4 \\
weight decay & 0.05 \\
layer-wise lr decay & 0.65 \\
GPUs & 8 \\
batch\_size & 96/768 \\
warmup epochs & 5 \\
training epochs & 100 \\
learning rate schedule & cosine decay \\
distillation weight & 1.0 \\
distillation temp. & 2.0 \\
augmentation & RandomAug~(9, 0.5) \\
label smoothing & 0.1 \\
mixup & 0.8 \\
cutmix & 1.0 \\
drop path & 0.1 \\
\end{tabular}
\end{table}

\begin{table}[t]
\caption{The hyperparameter of end-to-end finetuning with ViT-L ReLU on ImageNet.}
\label{tab:vit-l_ft_imagenet} 
\centering
\begin{tabular}{l|l}
config & value~(ViT-L ReLU)) \\ \hline
optimizer & AdamW \\
optimizer momentum & $\beta_1,\beta_2$ = 0.9, 0.99 \\
base learning rate & 1.67e-3 \\
weight decay & 0.05 \\
layer-wise lr decay & 0.75 \\
GPUs & 8 \\
batch\_size & 24/192 \\
warmup epochs & 5 \\
training epochs & 50 \\
learning rate schedule & cosine decay \\
distillation weight & 1.0 \\
distillation temp. & 2.0 \\
augmentation & RandomAug~(9, 0.5) \\
label smoothing & 0.1 \\
mixup & 0.8 \\
cutmix & 1.0 \\
drop path & 0.2 \\
\end{tabular}
\end{table}

\begin{table}[t]
\caption{The hyperparameter of Quantization-Aware-Training with ViT-S ReLU on ImageNet.}
\label{tab:vit-s_qat_imagenet} 
\centering
\begin{tabular}{l|l}
config & value~(ViT-S ReLU)) \\ \hline
optimizer & AdamW \\
optimizer momentum & $\beta_1,\beta_2$ = 0.9, 0.99 \\
base learning rate & 1.5e-4 \\
weight decay & 0.05 \\
layer-wise lr decay & 0.65 \\
GPUs & 4 \\
batch\_size & 64/256 \\
warmup epochs & 5 \\
training epochs & 100 \\
learning rate schedule & cosine decay \\
distillation weight & 1.0 \\
distillation temp. & 2.0 \\
augmentation & RandomAug~(9, 0.5) \\
label smoothing & 0.1 \\
mixup & 0.8 \\
cutmix & 1.0 \\
drop path & 0.1 \\
quantization level & 32 \\
\end{tabular}
\end{table}

\begin{table}[t]
\caption{The hyperparameter of Quantization-Aware-Training with ViT-B ReLU on ImageNet.}
\label{tab:vit-b_qat_imagenet} 
\centering
\begin{tabular}{l|l}
config & value~(ViT-B ReLU)) \\ \hline
optimizer & AdamW \\
optimizer momentum & $\beta_1,\beta_2$ = 0.9, 0.99 \\
base learning rate & 1.5e-4 \\
weight decay & 0.001 \\
layer-wise lr decay & 0.65 \\
GPUs & 6 \\
batch\_size & 96/576 \\
warmup epochs & 5 \\
training epochs & 50 \\
learning rate schedule & cosine decay \\
distillation weight & 1.0 \\
distillation temp. & 2.0 \\
augmentation & RandomAug~(9, 0.5) \\
label smoothing & 0.1 \\
mixup & 0.8 \\
cutmix & 1.0 \\
drop path & 0.05 \\
quantization level & 32 \\
\end{tabular}
\end{table}

\begin{table}[t]
\caption{The hyperparameter of Quantization-Aware-Training with ViT-L ReLU on ImageNet.}
\label{tab:vit-l_qat_imagenet} 
\centering
\begin{tabular}{l|l}
config & value~(ViT-L ReLU)) \\ \hline
optimizer & AdamW \\
optimizer momentum & $\beta_1,\beta_2$ = 0.9, 0.99 \\
base learning rate & 1.6e-3 \\
weight decay & 0.0005 \\
layer-wise lr decay & 0.75 \\
GPUs & 8 \\
batch\_size & 20/160 \\
warmup epochs & 5 \\
training epochs & 50 \\
learning rate schedule & cosine decay \\
distillation weight & 1.0 \\
distillation temp. & 2.0 \\
augmentation & RandomAug~(9, 0.5) \\
label smoothing & 0.1 \\
mixup & 0.8 \\
cutmix & 1.0 \\
drop path & 0.05 \\
quantization level & 32 \\
\end{tabular}
\end{table}

\subsection{Implementation Details on ImageNet}

The hyperparameter of our end-to-end finetuning with ViT-S ReLU, ViT-B ReLU and ViT-L ReLU on ImageNet are tabulated in~\cref{tab:vit-s_ft_imagenet}, \cref{tab:vit-b_ft_imagenet} and \cref{tab:vit-l_ft_imagenet}. Then we follow the hyperparameter in~\cref{tab:vit-s_qat_imagenet}, \cref{tab:vit-b_qat_imagenet} and \cref{tab:vit-l_qat_imagenet} to conduct 32-Level Quantization-Aware-Training~(QAT) on ViT-S ReLU, ViT-B ReLU and ViT-L ReLU obtained above, and achieve QViT-S-32Level, QViT-B-32Level and QViT-L-32Level in~\cref{tab:imagenet}, note that we introduce knowledge distillation which is effective on multiple tasks~\cite{zhou2021rethinking, xu2022delving}. Finally we apply our SpikeZIP-TF on QViT-S-32Level, QViT-B-32Level and QViT-L-32Level to achieve corresponding SViT-S-32Level, SViT-B-32Level and SViT-L-32Level in~\cref{tab:imagenet}.

\begin{table}[t]
\caption{The hyperparameter of end-to-end finetuning with ViT-S ReLU on CIFAR10/100.}
\label{tab:vit-s_ft_cifar} 
\centering
\begin{tabular}{l|l}
config & value~(ViT-S ReLU)) \\ \hline
optimizer & AdamW \\
optimizer momentum & $\beta_1,\beta_2$ = 0.9, 0.99 \\
base learning rate & 1e-4 \\
weight decay & 0.05 \\
layer-wise lr decay & 0.65 \\
GPUs & 8 \\
batch\_size & 192/1536 \\
warmup epochs & 5 \\
training epochs & 100 \\
learning rate schedule & cosine decay\\
distillation weight & 1.0 \\
distillation temp. & 2.0 \\
augmentation & RandomAug~(9, 0.5) \\
label smoothing & 0.1 \\
mixup & 0.8 \\
cutmix & 1.0 \\
drop path & 0.1 \\
\end{tabular}
\end{table}

\begin{table}[t]
\caption{The hyperparameter of Quantization-Aware-Training with ViT-S ReLU on CIFAR10/100.}
\label{tab:vit-s_qat_cifar} 
\centering
\begin{tabular}{l|l}
config & value~(ViT-S ReLU)) \\ \hline
optimizer & AdamW \\
optimizer momentum & $\beta_1,\beta_2$ = 0.9, 0.99 \\
base learning rate & 1.5e-4 \\
weight decay & 0.05 \\
layer-wise lr decay & 0.65 \\
GPUs & 8 \\
batch\_size & 128/1024 \\
warmup epochs & 5 \\
training epochs & 300 \\
learning rate schedule & cosine decay \\
distillation weight & 1.0 \\
distillation temp. & 2.0 \\
augmentation & RandomAug~(9, 0.5) \\
label smoothing & 0.1 \\
mixup & 0.8 \\
cutmix & 1.0 \\
drop path & 0.1 \\
quantization level & 8, 16 \\
\end{tabular}
\end{table}

\subsection{Implementation Details on CIFAR10/100}

The hyperparameter of our end-to-end finetuning with ViT-S ReLU on CIFAR10/100 are tabulated in \cref{tab:vit-s_ft_cifar}. Then we follow the hyperparameter in \cref{tab:vit-s_qat_cifar} to conduct 8-Level and 16-Level Quantization-Aware-Training (QAT) on ANN ViT-S ReLU obtained above, and achieve QViT-S-8Level and QViT-S-16Level in \cref{tab:cifar}. Finally we apply our SpikeZIP-TF on QViT-S-8Level and QViT-S-16Level to achieve corresponding 16 time-steps SViT-S and 32 time-steps SViT-S in \cref{tab:cifar}.

\begin{table}[t]
\caption{The hyperparameter of end-to-end finetuning with ViT-S ReLU on CIFAR10-DVS.}
\label{tab:vit-s_ft_cifar10dvs} 
\centering
\begin{tabular}{l|l}
config & value~(ViT-S ReLU)) \\ \hline
optimizer & AdamW \\
optimizer momentum & $\beta_1,\beta_2$ = 0.9, 0.99 \\
base learning rate & 2e-4 \\
weight decay & 0.05 \\
layer-wise lr decay & 0.65 \\
GPUs & 4 \\
batch\_size & 192/768 \\
warmup epochs & 5 \\
training epochs & 300 \\
learning rate schedule & cosine decay\\
distillation weight & 1.0 \\
distillation temp. & 2.0 \\
augmentation & RandomAug~(9, 0.5) \\
label smoothing & 0.1 \\
mixup & 0.8 \\
cutmix & 1.0 \\
drop path & 0.1 \\
\end{tabular}
\end{table}

\begin{table}[t]
\caption{The hyperparameter of Quantization-Aware-Training with ViT-S ReLU on CIFAR10-DVS.}
\label{tab:vit-s_qat_cifar10dvs} 
\centering
\begin{tabular}{l|l}
config & value~(ViT-S ReLU)) \\ \hline
optimizer & AdamW \\
optimizer momentum & $\beta_1,\beta_2$ = 0.9, 0.99 \\
base learning rate & 2.25e-4 \\
weight decay & 0.05 \\
layer-wise lr decay & 0.65 \\
GPUs & 4 \\
batch\_size & 92/368 \\
warmup epochs & 5 \\
training epochs & 300 \\
learning rate schedule & cosine decay \\
distillation weight & 1.0 \\
distillation temp. & 2.0 \\
augmentation & RandomAug~(9, 0.5) \\
label smoothing & 0.1 \\
mixup & 0.8 \\
cutmix & 1.0 \\
drop path & 0.1 \\
quantization level & 16, 32 \\
\end{tabular}
\end{table}

\subsection{Implementation Details on CIFAR10-DVS}

The hyperparameter of our end-to-end finetuning with ViT-S ReLU on CIFAR10-DVS are tabulated in \cref{tab:vit-s_ft_cifar10dvs}. Then we follow the hyperparameter in \cref{tab:vit-s_qat_cifar10dvs} to conduct 16-Level and 32-Level Quantization-Aware-Training (QAT) on ANN ViT-S ReLU obtained above, and achieve QViT-S-16Level and QViT-S-32Level in \cref{tab:cifar}. Finally we apply our SpikeZIP-TF on QViT-S-16Level and QViT-S-32Level to achieve corresponding 32 time-steps SViT-S and 64 time-steps SViT-S in \cref{tab:cifar}.

\end{document}